\documentclass[letterpaper]{article} 
\usepackage{aaai25}  
\usepackage{times}  
\usepackage{helvet}  
\usepackage{courier}  
\usepackage[hyphens]{url}  
\usepackage{graphicx} 
\usepackage{natbib}  
\usepackage{caption} 
\usepackage{algorithm}
\usepackage{algorithmic}

\usepackage{amsmath,amsthm,amssymb}
\usepackage{amsfonts,bm}
\usepackage{xr} %
\usepackage[capitalize,noabbrev]{cleveref}
\usepackage{subfigure}
\usepackage{braket}
\usepackage{booktabs,multirow}

\theoremstyle{plain}
\newtheorem{theorem}{Theorem}
\newtheorem{proposition}[theorem]{Proposition}

\theoremstyle{definition}
\newtheorem{definition}[theorem]{Definition}
\theoremstyle{remark}

\DeclareMathOperator*{\argmin}{arg\,min}

\DeclareMathOperator*{\trace}{tr}

\DeclareMathOperator*{\covv}{cov}

\newcommand{\reals}{\mathbb{R}}
       
\newcommand{\maxn}[1][X]{\|#1\|_{\mathrm{max}}}

\newcommand{\twopartdef}[4]
{\left\{
    \begin{array}{ll}
      #1 & \mbox{if } #2 \\
      #3 & \mbox{ } #4
    \end{array}
  \right.
}

\newcommand{\trs}[1][x]{{#1}^{\mathrm{T}}}

\newcommand{\grp}{G}%
\newcommand{\gre}{E}%

\DeclareMathOperator{\pa}{\mathrm{Pa}}
\DeclareMathOperator{\chs}{\mathrm{Ch}}

\newcommand{\dom}[1][d]{\reals^{#1\times #1}} 
 
\DeclareMathOperator{\degr}{deg}  
\DeclareMathOperator{\nnz}{nnz}
\def\X{\mbox{\bf X}}

\def\mbo{\mbox{\bf MB}}
\def\pcs{\mbox{\bf PC}}
\def\pss{\mbox{\bf P}}
\def\chs{\mbox{\bf Ch}}

\newcommand{\eg}{e.g.}
\newcommand{\ie}{i.e.}

\newcommand{\ind}{\perp\!\!\!\!\perp} 
\newcommand{\ttau}{\tau_{\mathrm{98}}} 
\newcommand{\comm}[1]{}
\newcommand{\checkk}[1]{}
\newcommand{\hcheck}[1]{} 

\usepackage{multicol} 

\usepackage{listings}
\usepackage{xcolor}
\definecolor{codegreen}{rgb}{0,0.6,0}
\definecolor{codegray}{rgb}{0.5,0.5,0.5}
\definecolor{codepurple}{rgb}{0.58,0,0.82}
\definecolor{backcolour}{rgb}{0.95,0.95,0.92}
\lstdefinestyle{mystyle}{
    commentstyle=\color{codegreen},
    keywordstyle=\color{magenta},
    numberstyle=\tiny\color{codegray},
    stringstyle=\color{codepurple},
    basicstyle=\ttfamily\footnotesize,
    breakatwhitespace=false,         
    breaklines=true,                 
    captionpos=b,                    
    keepspaces=true,                 
    numbers=left,                    
    numbersep=5pt,                  
    showspaces=false,                
    showstringspaces=false,
    showtabs=false,                  
    tabsize=2
}
\newcommand{\si}[1][]{\mbox{\bf S}_{#1}}
\newcommand{\xdata}{\mathcal{X}}

\def\XX{DCILP}

\newcommand{\phao}{Phase-1}
\newcommand{\phad}{Phase-2}
\newcommand{\phat}{Phase-3}

\newcommand{\PD}{Phase-2}
\newcommand{\PT}{Phase-3}

\newcommand{\HAI}{\widehat{A}^{(i)}}

\newcommand{\redl}[1]{#1}
\newcommand{\bleu}[1]{#1}

\title{\XX: A Distributed Approach for Large-Scale Causal Structure Learning}
\author {
    Shuyu Dong\textsuperscript{\rm 1}\equalcontrib,  
    Mich\`ele Sebag\textsuperscript{\rm 1}\equalcontrib,  
    Kento Uemura\textsuperscript{\rm 2},  
    Akito Fujii\textsuperscript{\rm 2},  
    Shuang Chang\textsuperscript{\rm 2},\\
    Yusuke~Koyanagi\textsuperscript{\rm 2},  
    Koji Maruhashi\textsuperscript{\rm 2} 
}
\affiliations {
    \textsuperscript{\rm 1}INRIA, LISN, Universit\'e Paris-Saclay, 
     91190, Gif-sur-Yvette, France\\
    \textsuperscript{\rm 2}Fujitsu Limited, Kanagawa, 211-8588, Japan\\
    shuyu.dong@inria.fr, michele.sebag@lri.fr, \{uemura.kento, fujii.akito, chang.shuang, koyanagi.yusuke, maruhashi.koji\}@fujitsu.com
}

\begin{document}

\maketitle

\begin{abstract}
Causal learning tackles the computationally demanding task of estimating causal graphs. This paper introduces a new divide-and-conquer approach for causal graph learning, called \XX. In the divide phase, the Markov blanket $\mbo(X_i)$ of each variable $X_i$ is identified, and causal learning subproblems associated with each  $\mbo(X_i)$ are independently addressed in parallel. This approach benefits from a more favorable ratio between the number of data samples and the number of variables considered. In counterpart, it can be adversely affected by the presence of hidden confounders, as variables external to $\mbo(X_i)$ might influence those within it.
The reconciliation of the local causal graphs generated during the divide phase is a challenging combinatorial optimization problem, especially in large-scale applications. 
The main novelty of \XX\ is an original formulation of this reconciliation as an integer linear programming (ILP) problem, which can be delegated and efficiently handled by an ILP solver. 
Through experiments on medium to large scale graphs, and comparisons with state-of-the-art methods, \XX\ demonstrates significant improvements in terms of computational complexity, while preserving the learning accuracy on real-world problem and suffering at most a slight loss of accuracy on synthetic problems.
\end{abstract}

%
\begin{links}
    \link{Code}{https://github.com/shuyu-d/dcilp-exp}
    \link{Extended version}{https://arxiv.org/abs/2406.10481}
\end{links}

\section{Introduction}\label{sec1}
Discovering causal relations from observational data emerges as an important problem for artificial intelligence with fundamental and practical motivations~\citep{Pearl2000,peters2017elements-brief}. 
One notable reason is that causal models support modes of reasoning, \eg, counterfactual reasoning and algorithmic recourse \citep{WshopAlgorithmicRecourseICML21}, that are beyond the reach of correlation-based models \citep{peters2016causal,arjovsky2019invariant,Sauer2021ICLR}.
In the literature of causal discovery and Bayesian network learning, there are two main categories of methods, namely constraint-based methods \citep{spirtes2000causation,meek1995causal} and score function-based methods \citep{chickering2002learning,loh2014high} 
(more in Section \ref{sec:SoA}).
Depending on the specific method, strategies for learning large causal graphs include restricting the search space of directed graphs to that of sparse graphs~\citep{ramsey2017million,loh2014high}, or transforming the underlying combinatorial problem into a continuous optimization problem \citep{NEURIPS2018_e347c514,aragam2019globally,ng2020role,ng2021reliable,lopez2022large}. While these strategies have resulted in significant improvements in reducing the complexity, their scalability is still limited when the number of variables and/or the degree of the sought causal graph are high.

To better tackle the computational challenges in learning large causal structures, a growing number of works consider breaking down the large-scale causal discovery problem into smaller ones 
defined from subsets of variables, and conducting a divide-and-conquer strategy. 
Such subsets of variables might be incrementally built and refined \citep{gao2017local}; they can be based on hierarchical clustering \citep{gu2020learning}, recursive decomposition based on conditional independence tests \citep{zhang2020learning}, or through the Markov blanket (defined in \Cref{sec:formal}) associated with each variable \citep{tsamardinos2003algorithms,wu2020accurate,wu2022multi,wu2023practical,mokhtarian2021recursive}. A main challenge lies in the conquer step for the fusion or reconciliation of the partial solutions identified in the divide step; most conquer approaches are rule-based which limits their applicability. 

In this paper we present the original {\em Divide-and-Conquer causal modelling with Integer Linear programming}  approach (\XX) to address the scalability challenge inherent to causal discovery. Formally, 
\XX\ consists of three phases: 
\begin{itemize}
    \item 
\phao\ estimates the Markov blankets  $\mbo(X_i)$ associated with each observed variable $X_i$, that are used to divide the causal discovery problem into smaller subproblems; \item \phad\ independently tackles each causal subproblem defined by the variables in  $\mbo(X_i)$. Note that these causal subproblems offer a more favorable ratio between the number of samples and the number of variables; in counterpart, they may no longer satisfy the assumption of causal sufficiency (section \ref{sec:formal}) --- as variables external to a Markov blanket might act as hidden confounders, having an impact on the variables within it; 
\item \phat\ is a conquer phase, reconciling the causal relations identified for each subproblem into a globally consistent causal graph. \end{itemize}
 
The original contribution of the proposed \XX\ is twofold. 
Firstly, \phad\ is parallelizable by design; the fact that it can handle the causal discovery subproblems associated with each Markov blanket, allows it to scale up to a few thousand variables. The causal insufficiency issue is mitigated by only retaining the causal relations involving the center variable of the Markov Blanket. 
Secondly, and most importantly, we show that the reconciliation of the causal subgraphs learned in \phad\ can be formulated and efficiently solved as an integer linear programming (ILP) problem.
Binary ILP variables are defined to represent the causal relations (causes, effects, spouses, and v-structures); logical constraints are defined to
enforce their consistency, and the optimization of the ILP variables aims to find a causal graph as close as possible to 
the concatenation of all local subgraphs, subject to the consistency constraints.
The resolution of this ILP problem can be delegated to highly efficient ILP solvers. 

%
Overall, \XX\ defines a flexible framework where different algorithmic components can be used in each phase: 
%
(i) For the Markov blanket discovery task in \phao, we restrict ourselves to linear structural equation models (\Cref{sec:formal}) following the setting of \cite{loh2014high};\footnote{The case of non-linear causal relations is left for further work, using for instance feature selection to elicit Markov blankets \citep{tsamardinos2003algorithms}.} 
(ii) for the causal discovery subproblems in \phad, we consider GES~\citep{chickering2002optimal} and DAGMA \citep{bello2022dagma} as they are two representative and efficient state-of-the-art algorithms for causal modeling; %
(iii) for \phat, we use the Gurobi ILP solver \citep{gurobi}. 

The paper is organized as follows. After presenting the formal background in \Cref{sec:formal}, we describe \XX\ in \Cref{sec:main}. Section \ref{sec:exp} %
presents the experimental setting and results of \XX. %
Section \ref{sec:SoA} discusses the position of \XX\ with respect to the related work. Section \ref{sec:conclu} concludes the paper and presents some perspectives for further work.
\section{Formal Background} 
\label{sec:formal} 

\begin{definition}
Let $\X = (X_1,\dots,X_d)$ be a set of $d$ observed random variables. The linear Structural Equation Model (SEM) of $\X$  is a set of $d$ equations:
\begin{eqnarray*}
\forall 1 \le i \le d, & X_i  =  \beta_{1,i} X_1 + \dots + \beta_{d,i} X_d + \epsilon_i\nonumber     
\end{eqnarray*}
with $\epsilon_i$ an external random variable, independent of any $X_j$ for $j \neq i$.
Coefficient $\beta_{i,i} = 0$; at most one of $\beta_{i,j}$ and $\beta_{j,i}$ is nonzero. 
If $\beta_{j,i} \neq 0$, $X_j$ is said to be a cause, or parent, of $X_i$; $X_i$ is said to be an effect of $X_j$.
\end{definition}

The graph $\grp:=(\X,\gre)$ with adjacency matrix $B =
(\beta_{i,j})$, referred to as causal graph, is the directed graph such that the edge set $\gre$ is the set of
pairs $(i,j)$ with $\beta_{i,j} \neq 0$. The directed graph of a linear SEM is usually
required to be acyclic (DAG).

\begin{definition}
Consider a DAG $\grp=(\X, B)$ defined on $\X=(X_1,\dots,X_d)$. 
The Markov blanket of variable $X_i$, denoted as $\mbo(X_i)$, is the smallest set $M \subset \X$ such that %
$$ X \ind_{\grp}  \X \backslash (M \cup \{X_i\}) \text{~given~} M,$$
where $\ind_{\grp}$ denotes $d$-separation \citep[Definition~6.1, Definition~6.26]{peters2017elements-brief}. 
\end{definition}
Causal learning commonly involves two assumptions. The {\em Markov property} states that the joint distribution of $\X$ exactly reflects the dependencies and independence relations in causal graph $G$, implying that 
$P(\X) = \prod_{i=1}^d P(X_i | \pa(X_i))$ 
with $\pa(X_i)$ 
denoting the set of variables parent of $X_i$ in $G$. 
The {\em Markov sufficiency} assumption states that all confounders (variables causes of at least two observed variables) are also observed. 

Under these assumptions, 
the Markov blanket $\mbo(X_i)$ is the smallest set $M$ such that $X_i$ is independent of all other variables in $\X$ given $M$; 
and $\mbo(X_i)$ contains 
exactly the variables $X_j$ that are causes or effects of $X_i$ (\ie,
$\beta_{j,i}\neq 0$ or $\beta_{i,j} \neq 0$) and the spouse variables $X_k$
(\ie, such that there exists a variable $X_\ell$ that is an effect of both $X_i$ and $X_k$). A triplet ($X_i, X_j, X_k$) form a v-structure \redl{if $X_i$ and $X_j$ are causes of $X_k$ while the first two are not directly linked.}

\def\SI{\mbox{\bf S}_{i}}
\def\SJ{\mbox{\bf S}_{j}}
\def\SK{\mbox{\bf S}_{k}}
\section{Overview of \XX} 
\label{sec:main}
After describing the divide-and-conquer strategy at the core of \XX, this section details the Integer Linear Programming approach used for the reconciliation of the local causal graphs. In the remainder of the paper, we assume the Markov property and causal sufficiency.

\subsection{Divide-and-Conquer Strategy}
As illustrated in Fig. \ref{fig:dc-illu}, \XX\ is a 3-phase process: 
\begin{itemize}
    \item \phao\ aims to identify the Markov blankets $\mbo(X_i)$ for each variable $X_i$ with $i\in[d]$. Under the assumption that it is accurately identified, $\mbo(X_i)$ contains all variables relevant to predict $X_i$ \citep{tsamardinos2003algorithms}, that is, the causes, effects and spouses of $X_i$. 
\item \phad\ tackles the local causal subproblems defined on $\SI:=\{X_i\} \, \cup \, \mbo(X_i)$, for $i\in[d]$. 
The restriction to the $\SI$s usually makes the causal subproblems much smaller than the overall causal discovery problem; in counterpart, they no longer satisfy the causal sufficiency assumption as variables in $\mbo(X_i)$ may be influenced by variables external to $\mbo(X_i)$. This issue will be partially mitigated by only retaining the causal relations involving $X_i$, i.e. the causes and effects of $X_i$. 
\item \phat\ tackles the building of a global causal graph based on the Markov blankets  and the cause-effect edges respectively found in \phao\ and \phad.
This task is formalized as an ILP problem. ILP binary variables are associated with the cause-effect, spouse and v-structure relations among the causal variables. ILP constraints express the consistency of these relations. Lastly, the ILP objective states that the overall causal graph should be aligned to the best extent with the local causal graphs, subject to the consistency constraints. 
\end{itemize}

\begin{figure}[htpb]
\centering
\raisebox{-0.5\height}{
\subfigure[]{\includegraphics[width=.061\textwidth]{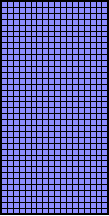}} 
}
\raisebox{-0.5\height}{
{\includegraphics[width=.028\textwidth]{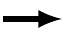}} 
}
\raisebox{-0.5\height}{
\subfigure[]{\includegraphics[width=.066\textwidth]{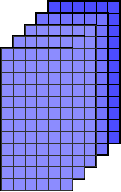}} 
}
\raisebox{-0.5\height}{
{\includegraphics[width=.028\textwidth]{figs/dc-ar.pdf}} 
}
\raisebox{-0.5\height}{
\subfigure[$\widehat{B}^{(i)}$]{\includegraphics[width=.074\textwidth]{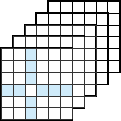}} 
}
\raisebox{-0.5\height}{
{\includegraphics[width=.028\textwidth]{figs/dc-ar.pdf}} 
}
\raisebox{-0.5\height}{
\subfigure[$B$]{\includegraphics[width=.074\textwidth]{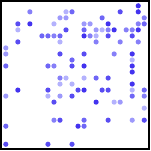}} 
}
\caption{\XX: %
(a) observational data; (b) data subsets on $\si[i]$; %
\bleu{(c) output matrix $\widehat{B}^{(i)}$ of \phad; (d) final solution $B$.}}
\label{fig:dc-illu}
\end{figure}

\XX\  (\Cref{alg:dc-ges}) defines a general framework, where different algorithms can be used in the three phases. \phao\ extracts the $\mbo(X_i)$s from the whole observational dataset $\cal D$. In \phad, the causal learning algorithm considers the local subproblems associated with the projection of $\cal D$ on the subset of variables $\si[i]$s. 
This phase can be parallelized, estimating the causal subgraphs related with the $\si[i]$ in an independent manner. Lastly, these relations are reconciled in \phat.

\begin{algorithm}[htpb] 
\caption{\XX } \label{alg:dc-ges}
\begin{algorithmic}[1]
\REQUIRE{Observational data $\xdata \in \reals^{n\times d}$   }
\STATE {\bf \textsc{(\phao)} Divide}: 
\begin{center}
    Estimate Markov blanket $\mbo(X_i)$ for $i\in[d]$ 
\end{center}
\STATE{\bf \textsc{(\phad)} for $i\in[d]$ do in parallel} 
\STATE \label{eq:localpb} \hspace{5mm} $A^{(i)} \leftarrow$ Causal discovery on $\si[i]:=\{X_i\}\,\cup\,\mbo(X_i)$ 
\STATE \label{eq:filt-ci} \hspace{5mm} $\widehat{B}^{(i)}_{j,k} \leftarrow A^{(i)}_{j,k}$ if $j=i$ or $k=i$, \text{~ and~} 0 otherwise \\
\STATE {\bf \textsc{(\phat)} Conquer:} 
\vspace{-2mm}
$$ B \leftarrow \text{Reconciliation from } \{\widehat{B}^{(i)}, i\in[d]\} \text{~via ILP} %
$$
\end{algorithmic} 
\end{algorithm}
\subsection{Conflicts Among Causal Subgraphs} 
The naive concatenation of the cause-effect relations found in \phad, given as: 
\begin{equation}
\label{eq:naive}
\widehat{B} = \sum_{i=1}^d \widehat{B}^{(i)}. 
\end{equation}
generally fails to be a DAG, because $\widehat{B}^{(i)}$ might include undirected edges, and because some conflicts might exist among different $\widehat{B}^{(i)}$s.\\

\begin{definition}
    [Conflicts among causal subgraphs]
    \label[definition]{defi:mergeco}
    Let $B^{(1)}$ and $B^{(2)}$ be defined as binary adjacency matrices among variables respectively ranging in $V_1$ and $V_2$, such that $V_1 \subset \mathbf{X}$ and $V_2\subset \mathbf{X}$. There exists a conflict among $B^{(1)}$ and $B^{(2)}$ if and only if (iff) there exists a pair of variables $(X_i,X_j)$ such that $\{X_i,X_j\} \subset V_1 \cap V_2$ with $B^{(1)}_{ij} \neq B^{(2)}_{ij}$. 
\end{definition}

\begin{proposition}%
    \label[proposition]{def:mergeco}
    For $i, j\in [d]$, let $\widehat{B}^{(i)}$  and $\widehat{B}^{(j)}$ respectively denote the $i$-th and $j$-th binary adjacency matrix output by \phad\ (\Cref{alg:dc-ges}, l. 4). 
    A merge conflict \redl{between} $\widehat{B}^{(i)}$  and $\widehat{B}^{(j)}$
    is one of the following three types:
    \begin{enumerate}
    \item {\em (Type-1)} One of the two adjacency matrices includes a directed edge between $X_i$ and $X_j$ while the other includes no edge ($X_i \ind X_j$). 
    
    \item {\em (Type-2)} The two adjacency matrices %
    include edges between $X_i$ and $X_j$ with opposite directions: either one matrix contains an undirected edge between $X_i$ and $X_j$ while the other includes no edge; or both matrices includes a directed edge with opposite direction. 
    
    \item {\em (Type-3)} One of the two adjacency matrices include an undirected edge between $X_i$ and $X_j$ while the other includes a directed edge. 
    \end{enumerate}
\end{proposition}
The proof is given in \cite[{Appendix~B}]{ArxivUS}.

As said, only the causal relations involving $X_i$ are retained in $\widehat{B}^{(i)}$ (\Cref{alg:dc-ges}). This restriction considerably reduces the set of conflicts compared to retaining all causal relations in $\HAI$ (\Cref{alg:dc-ges}, line 3).

When there exists a conflict between any pair of local solutions, and depending on its type, the naive concatenation $\widehat{B}$ (Eq. \ref{eq:naive}) is weighted accordingly; details in \cite[{Appendix~C.2}]{ArxivUS}.

\subsection{Reconciling Causal Subgraphs Through ILP}
\label{sssec:conquer}
\label{parag:ilp}
The reconciliation of the $\widehat{B}^{(i)}$s found in  \phad\ to form a consistent causal graph is formulated as an integer linear programming problem~\citep{wolsey2020integer}. 
Let us denote $B_{ij}$ the binary constraint variable representing the causal relation $X_i \to X_j$; let similarly $S_{ij}$ denote the binary constraint variable representing the spouse relation between $X_i$ and $X_j$. Finally, binary constraint variable $V_{ijk}$ represent the v-structure $X_i \to X_k \leftarrow X_j$ among  causal variables $X_i$, $X_j$ and $X_k$: 
\begin{align*}
    B_{ij}  = 1  & \text{~iff~}  X_i \to X_j \\ 
    V_{ijk} = V_{jik} = 1 & \text{~iff there is a v-structure $(X_i$ \hspace{-0.5mm}{\footnotesize $\to$} $X_k$ {\footnotesize$\leftarrow$} \hspace{-0.5mm}$X_j)$ } \\
    S_{ij}  = S_{ji} = 1  & \text{~iff $X_i$ and $X_j$ are spouses: $\exists k$, $V_{ijk} = 1$}. 
\end{align*}

\def\mbi{$\mbo(X_i)$}
\def\mbj{$\mbo(X_j)$}
\def\mbk{$\mbo(X_k)$}
Let variables $B_{ij}, S_{ij}$ and $V_{ijk}$ be referred to as ILP variables for clarity in the following. 
The ILP variables are created if and only if they are consistent with the Markov blankets, that is, $B_{ij}$ and $S_{ij}$ are created if $X_i \in \mbo(X_j)$ and $X_j \in \mbo(X_i)$; and $V_{ijk}$ is created if 
$X_i \in \mbo(X_j) \cap \mbo(X_k)$, $X_j \in \mbo(X_i) \cap \mbo(X_k)$, $X_k \in \mbo(X_i) \cap \mbo(X_j)$. 

The following constraints defined on the ILP variables express the logical relations among %
Markov blankets and v-structures, for all $i,j,k$ such that $i\neq j$, $j\neq k$, $k\neq i$:  
\begin{align}
& B_{ij} + B_{ji} + S_{ij} \geq 1  \hspace*{.19in}  \text{~if~} \{i,j\}\subset (\si[i]\cap\si[j]) \label{eq:mb} \\ %
& V_{ijk} \leq B_{ik}, V_{ijk} \leq B_{jk} \hspace*{.07in} \text{if~} \{i,j,k\} \subset (\si[i] \cap \si[j] \cap \si[k])  \label{eq:v-str1} \\
& B_{ik} + B_{jk} \leq 1 + V_{ijk} \hspace*{.12in} 
\text{~if~} \{i,j,k\} \subset (\si[i] \cap \si[j] \cap \si[k])
\label{eq:v-str3}\\
& V_{ijk} \leq S_{ij}  \hspace*{.78in} 
\text{~if~} \{i,j,k\} \subset (\si[i] \cap \si[j] \cap \si[k])
\label{eq:v-str2} \\ 
& S_{ij} \leq \sum_{k} V_{ijk} \hspace*{.56in} 
\text{~if~} \{i,j,k\} \subset (\si[i] \cap \si[j] \cap \si[k]) 
\label{eq:v-str4}
\end{align}

\begin{proposition}%
    \label[proposition]{def:co-MB}
With same notations as above, under the assumption that Markov blanket \mbi\ is true for all $i\in[d]$, the ILP variables associated with the sought causal graph $B$ satisfy all linear constraints \eqref{eq:mb}--\eqref{eq:v-str4}. 
\end{proposition}
\begin{proof}
Constraint \eqref{eq:mb} expresses that if $X_j$ belongs to \mbi\ it is either its cause, or its effect, or a spouse.
The fact that $X_i \to X_k \leftarrow X_j$ is equivalent to constraints \eqref{eq:v-str1}, \eqref{eq:v-str2}, \eqref{eq:v-str3} 
results 
from the truth table of the involved ILP variables.  Constraint \eqref{eq:v-str4} expresses that if $X_i$ and $X_k$ are spouses, then there exists at least one $k$ such that $i,j,k$ forms a v-structure $X_i \to X_k \leftarrow X_j$.    
\end{proof}

The reconciliation of the causal subgraphs $\widehat{B}^{(i)}$ for $i \in [d]$ is formulated as the following ILP problem which searches $(B,S,V)$ for $B$ to be aligned with the naive concatenation $\widehat{B}$ of the causal subgraphs (Eq. \ref{eq:naive}) subject to the consistency constraints:
\begin{align} 
\label{prog:the-ilp}
    \max_{B,S,V} \braket{\widehat{B}, B} \text{ subject to constraints \eqref{eq:mb}--\eqref{eq:v-str4}}.
\end{align}

\subsection{Refining the ILP Formulation}
\label{ssec:ilp-formu}
For tractability, we further leverage the output $\widehat{B}^{(i)}$s of \phad, by imposing $B_{ij} = 0$ if 
no causal relations between $X_i$ and $X_j$ are found in  $\widehat{B}^{(i)}$ or $\widehat{B}^{(j)}$:
\begin{equation}
 B_{ij} =0 \quad \mbox{~if~} \quad \widehat{B}^{(i)}_{ij} = \widehat{B}^{(i)}_{ji} =\widehat{B}^{(j)}_{ij} = \widehat{B}^{(j)}_{ji} = 0 
 \label{eq:bb}
\end{equation}
(\ie, if $\widehat{B}_{ij} = \widehat{B}_{ji} = 0$ in Eq. \ref{eq:naive}). 
The %
constraints (\ref{eq:bb}) and (\ref{eq:v-str1}) imply that, for all $i,j,k$ mutually distinct: 
\begin{align}
\label{eq:vij-zeros}
 V_{ijk}=0 \quad \text{if } \widehat{B}_{ik}=0 \text{ or } \widehat{B}_{jk} = 0,
\end{align}
which further reduce the set of the ILP variables. 
Lastly, we add constraints to forbid 2-cycles: 
\begin{equation}
B_{ij} + B_{ji} \leq 1 \quad \text{if~} \{i,j\}\subset (\si[i]\cap\si[j]) \label{eq:2cyc}    
\end{equation}
Therefore \phat\ tackles the ILP problem (\ref{prog:the-ilp}) augmented with constraints \eqref{eq:bb}, \eqref{eq:vij-zeros} and \eqref{eq:2cyc}. 

We note that the ILP resolution selects by design the edge orientations consistent with the Markov blankets (Fig. \ref{fig:exa1et2}). \XX\ thus recovers some properties of the PC algorithm \citep{spirtes2000causation}, with the difference that it exploits the identification of the Markov blankets, as opposed to conditional independence tests.

\begin{figure}[htpb]
\centering
\subfigure[$\widehat{B}$]{\includegraphics[width=.14\textwidth]{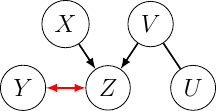} } 
\quad
\subfigure[$B$]{\includegraphics[width=.14\textwidth]{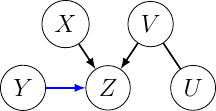} }
~ 
\subfigure[$B'$]{\includegraphics[width=.14\textwidth]{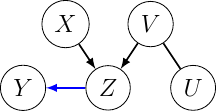} }
\caption{Consistency constraints based on Markov blankets information: if $Y$ belongs to $\mbo(V)$, the only feasible solution is $B$; otherwise it is $B'$.} 
\label{fig:exa1et2}
\end{figure}

The ILP resolution offers additional possibilities regarding the selection of the overall causal graph. 
In the implementation of \XX, Gurobi's ILP solver \citep{gurobi} is used for solving the ILP problem, and it enables returning a pool of ILP solutions attaining the same optimal value of the objective function. Then we select the solution with the best {\it DAGness value} 
based on \cite[Theorem~1]{NEURIPS2018_e347c514}: the DAGness of $B$ is defined as 
$h (B) = \trace(\exp(B\odot B))$ 
satisfying $h(B) = d$ iff $B$ is a DAG, and $h(B) > d$ otherwise. In the case where several ILP solutions have same $h$ value, the sparsest one (with fewer edges) is selected.

\subsection{Discussion}
\label{sec:discuss}
A main limitation of the presented framework is that errors in \phao\ (e.g. variables missing or spuriously added in the Markov blankets) adversely affect \phad\ and \phat. 
The impact of a missing variable $X_j$ in $\mbo(X_i)$ can be mitigated if $X_i$ is present in $\mbo(X_j)$;\footnote{\phao\ currently determines the Markov blankets from the precision matrix of $\X$, which is symmetric; see \cite[{Appendix C.1}]{ArxivUS}.} otherwise, the edge cannot be recovered. 
On the contrary, the presence of spurious variables in $\mbo(X_i)$ can to some extent be compensated for in \phad\ and \phat. 
Another limitation is that the ILP solution is not guaranteed to be a DAG; constraints (\ref{eq:2cyc}) only rule out 2-cycles. The ILP formulation however makes it possible to incrementally add new constraints, to rule out the cycles present in former solutions. The question is whether the incremental addition of such constraints yield a DAG solution in a (sufficiently) small number of iterations.
These limitations can be addressed by choosing appropriate algorithms involved in \phao\ or \phad. For instance, \phao\ currently exploits the precision matrix of $\X$ to identify the Markov blankets under the Gaussian linear SEM assumption. An alternative in non-linear cases is to leverage feature selection algorithms \citep{tsamardinos2003algorithms} to identify the smallest subset of variables needed to predict a given variable $X_i$.

Along the same line, the choice of the causal learning algorithm in \phad\ may depend on the prior knowledge about the data at hand, e.g., whether the causal variables are expected to satisfy the equal-variance assumption (EV case) or not (NV case). Some algorithms such as GES are more robust to the so-called varsortability bias~\cite{ReisachSW21} %
than others (e.g., NOTEARS and DAGMA \citep{bello2022dagma}).

\paragraph{Complexity.} In \phao, the identification of the Markov blankets is conducted under the assumption of Gaussian graphical models. %
Letting $\Sigma = \covv(\X)$ denote the covariance matrix of the data, and noting $\Theta = \Sigma^{-1}$ its inverse (the precision matrix), it follows from the Hammersley--Clifford theorem that the support of $\Theta_{i,:}$ is exactly $\SI = \mbo(X_i) \cup \{X_i\}$. 
Its complexity is bounded by $O(d^3)$, corresponding to the inversion of the symmetric positive definite matrix $\Sigma$. More advanced methods with reduced complexity can be used; see, \eg, \cite{hsieh2014quic}. 

In \phad, the causal subproblems, restricted to the subset of variables $\SI$, are solved independently for $i \in [d]$. Letting $d_{MB}$ denote the maximal size of the estimated Markov blankets, the computational complexity of \phad, considering that the subproblems are handled in parallel and using e.g. DAGMA \citep{bello2022dagma}, thus is $O(d_{MB}^3)$. %

In \phat, the number of ILP variables $B_{ij}$ and $S_{ij}$ is bounded by $O(d_{MB}d)$ with usually $d_{MB} \ll d$. The number of $V_{ijk}$ variables is bounded by $O(d_{MB}^2 d)$ in view of \eqref{eq:vij-zeros}. 
The ILP problem is NP-hard and becomes intractable when the number of variables and the size of the Markov blankets increases beyond a certain range. The observation from experimental results is that the computational gains in \phad\ (solving in parallel the $d$ subproblems as opposed to learning one causal graph of $d$ nodes) cancel out the computational cost of \phat, for $d$ ranging from medium to large-sized. 
\section{Experiments}
\label{sec:exp}
This section reports on the experimental validation of \XX, referring to the extended version~\cite{ArxivUS} for more details and complementary results. 
\def\DD{{\XX-dagma}} %
\def\DG{{\XX-ges}}
\subsection{Experimental Setting}

\paragraph{Goals.} The primary goal of the experiments is to evaluate the performance  of \XX\ according to the standard  SHD, TPR, FDR and FPR indicators for causal learning, together with its computational efficiency. 

A second goal is to assess how the causal learner used in \XX' \PD\ influences the overall performances. We report on the performances of \DG\ (respectively \DD), corresponding to \XX\ using GES (resp. DAGMA) during \phad. The choice of GES\footnote{We used the implementation from the R package \texttt{pcalg} (\url{https://cran.r-project.org/web/packages/pcalg/index.html}) , which is a most efficient version of GES to our knowledge.} \citep{chickering2002optimal} and DAGMA \citep{bello2022dagma}, %
known as a notable refinement of NOTEARS~\citep{NEURIPS2018_e347c514}, is used as they are two representative state-of-the-art causal learning methods using different techniques: GES is a greedy search method (optimal in the large sample limit) for finding the CPDAG while DAGMA is an efficient continuous optimization method for learning causal DAGs. 
\DG\ and \DD\ are assessed against the GES and DAGMA baselines; GOLEM~\citep{ng2020role} and DAS~\citep{montagna2023scalable} are also used for comparison.\footnote{Results of DAS are reported from \citep{montagna2023scalable}.} We also examine the impact of the \phao\ performance on the overall result, by considering the \XX\ (MB*) variants where the ground truth Markov blankets are supplied to \PD.

\paragraph{Benchmarks.}
Following \cite{NEURIPS2018_e347c514}, we consider synthetic and real-world datasets. The synthetic observational datasets are generated from linear SEMs, where the causal graph $B^\star$ is drawn from random DAGs with an Erd\H{o}s--R\'enyi (ER) 
 or scale-free (SF) model. The number $d$ of variables ranges in $\{$50, 100, 200, 400, 800, 1000, 1600$\}$. The number $n$ of samples is set to 
$5\times, 10\times$ or $50\times d$.
The real-world dataset is generated using the so-called 
MUNIN model~\citep{gu2020learning,munin-ref}, 
which is a Bayesian network with $d=$ 1,041 nodes that models a medical expert system based on electromyographs (EMG) to assist diagnosis of neuromuscular disorders. 
Datasets are generated according to the considered graph $B^\star$, with coefficients (edge weights) drawn from the uniform distribution $\text{Unif}([-2,-0.5]\cup [0.5, 2])$. The noise variables of the SEM data %
are from either Gaussian, Gumbel or uniform distribution. All reported results are median values from 
3 to 5 runs with different random seeds.

\paragraph{Computational environment.}
The parallelization of \XX' \PD\ is conducted on a cluster via the SLURM scheduler; each experiment uses at most 400 CPU cores. For $d \le 200$, \PD\ is distributed among $2d$ CPU cores; each local causal problem is handled on 2 cores. For higher values of $d$, the parallelization of \PD\ %
is then limited by congestion. 
The baseline algorithms and \XX' \phao\ run on one CPU core, and \XX' \PT\ runs on four CPU cores; CPU specifications are detailed in \cite[{Appendix~D.2}]{ArxivUS}. %

\subsection{Results on Synthetic Graphs}
\label{ssec:exp-scala}

\paragraph{\DG.}
\Cref{fig:exp1-dcges} and \Cref{tab:er1-withdas} illustrate the performances of \XX-ges and \DG~(MB*) %
on the ER2 data with $n=50d$ and on ER1 data with $n=10d$ respectively, 
in comparison with GES and the other baselines. %

On these problems, \XX\ outperforms GES, particularly in terms of FDR and computational time, except for small problem dimensions ($d \le 200$), as GES is very fast on smaller problems. \DG\ and \DG\ (MB*) yield identical results, as might be expected given the favorable $n/d$ ratio. The computational time for \DG\ increases with $d$, though at a much slower rate than for GES. The distribution of computational time among the \XX\ phases shows that the majority of the time is spent in \PD, increasing with $d$.
 
\begin{figure}[htpb]
    \centering
    \includegraphics[width=.472\textwidth]{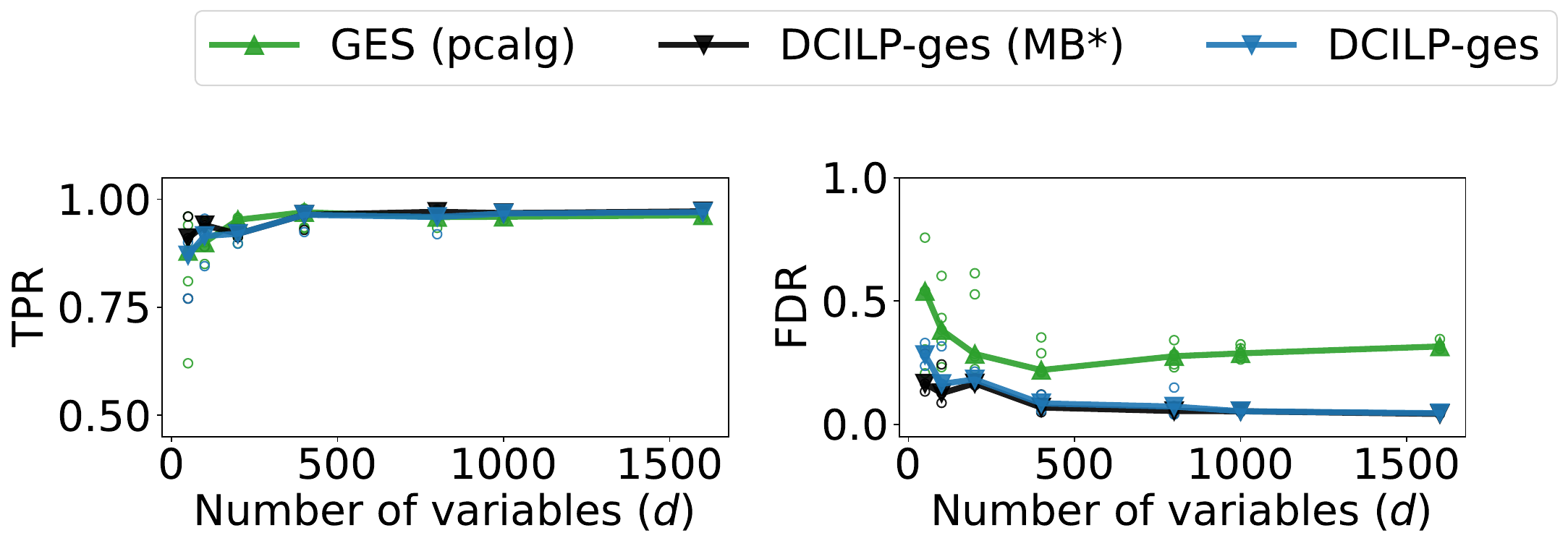}
    \caption{Learning accuracy of GES, \DG, and \DG\ (MB*) on ER2 data. }
    \label{fig:exp1-dcges}
\end{figure}

\DG\ achieves significant gains over GES in terms of DAG learning accuracy, which can be explained from two facts. Firstly, the causal learning subproblems in \XX' \PD\ enjoy a higher {\it effective sample size}, in the sense that the ratio $n /|\mbo(X_i)|$ is  higher than for the global problem. Secondly, the reconciliation achieved by \phat\ among the causal subgraphs tends to remove spurious directions and it avoids 2-cycles, while GES focuses on producing a CPDAG and might thus retain undirected edges.  

Fig. \ref{fig:runtimes} suggests that the speed-up achieved by \DG\ over GES increases for $d \ge 400$. For $d=1600$, \DG\ achieves more than $5\times$ speedup over GES on ER2 data. %

\begin{table*}[ht]
\small
\centering
\begin{tabular}{c|l||ccccr}
\hline
     $d$                                                                              & Algorithm            &  TPR        & FDR         &  SHD           & \# of edges      & time (sec)                                \\ 
\hline
\multicolumn{1}{c|}{\multirow{5}{*}{400}}     & GOLEM               & 0.987    	& 0.025   	& 46 {$\pm$ 14.47}   	& 1208          & 4437.526      \\ %
\multicolumn{1}{c|}{                   }    & GES                 & 0.907    	& 0.471   	& 1029 {$\pm$ 237.57}   	& 2048          & 1330.314        \\ 
\multicolumn{1}{c|}{                   }    & DAGMA               & 0.999    	& 0.003   	& 3 {$\pm$ 1.53}   	    & 1196          & 3997.719        \\ 
\multicolumn{1}{c|}{                   }    & DCILP-dagma (MB*)   & 0.964    	& 0.135   	& 208 {$\pm$ 24.43}   	& 1330        	& 1102.636        \\ 
\multicolumn{1}{c|}{                   }    & DCILP-dagma         & 0.925    	& 0.177   	& 322 {$\pm$ 68.83}   	& 1336        	& 974.538         \\ 
\hline
\multicolumn{1}{c|}{\multirow{5}{*}{800}}     & GOLEM               & 0.964    	& 0.049   	& 187 {$\pm$ 355.73}   	& 2426         & 25533.171   \\ %
\multicolumn{1}{c|}{                   }    & GES                 & 0.878    	& 0.518   	& 2444 {$\pm$ 1182.74}   & 4068         & 12063.320     \\ 
\multicolumn{1}{c|}{                   }    & DAGMA               & 0.999    	& 0.003   	& 7 {$\pm$ 8.66}   	    & 2398         & 22952.589     \\ 
\multicolumn{1}{c|}{                   }    & DCILP-dagma (MB*)   & 0.938    	& 0.250   	& 879 {$\pm$ 234.64}   	& 2994         & 4104.407      \\ 
\multicolumn{1}{c|}{                   }    & DCILP-dagma         & 0.923    	& 0.250   	& 914 {$\pm$ 191.78}   	& 2932         & 3405.801      \\ 
\hline
\end{tabular}
\caption{SF3 ($d \in \{400, 800\}, n/d = 20$): Comparative performance of GOLEM, DAGMA, GES, and \DD. \DD\ uses 400 CPU cores in \phad.} 
\label{tab:sf3-one}
\end{table*}

\paragraph{\bf \DD.}
\Cref{tab:sf3-one} and \Cref{tab:er1-withdas} show the performances of \XX-dagma and \DD\ (MB*) on SF3 data (for $n/d = 20$) and ER1 data (for $n/d = 10$) respectively, in comparison with DAGMA and the other baselines. %

We observe that DAGMA almost exactly recovers the underlying DAGs, with TPR and FDR respectively close to 1 and 0. \DD\ achieves significant gains over DAGMA in running time 
for all problem dimensions $d$,
with a moderate loss in learning accuracy (median TPR circa 0.9; median FDR around 0.2 on SF3 data and under 0.1 on ER1 data). %
Complementary results in the non-equal noise variance (NV) setting \cite{ReisachSW21} show however that \XX-ges is much more robust in the NV case than DAGMA (\Cref{fig:bplot-ev-nv}). More details are presented in the extended version~\cite{ArxivUS}. 

\begin{figure}[htpb]
    \centering
    \subfigure[DCILP-ges]{\includegraphics[width=.235\textwidth]{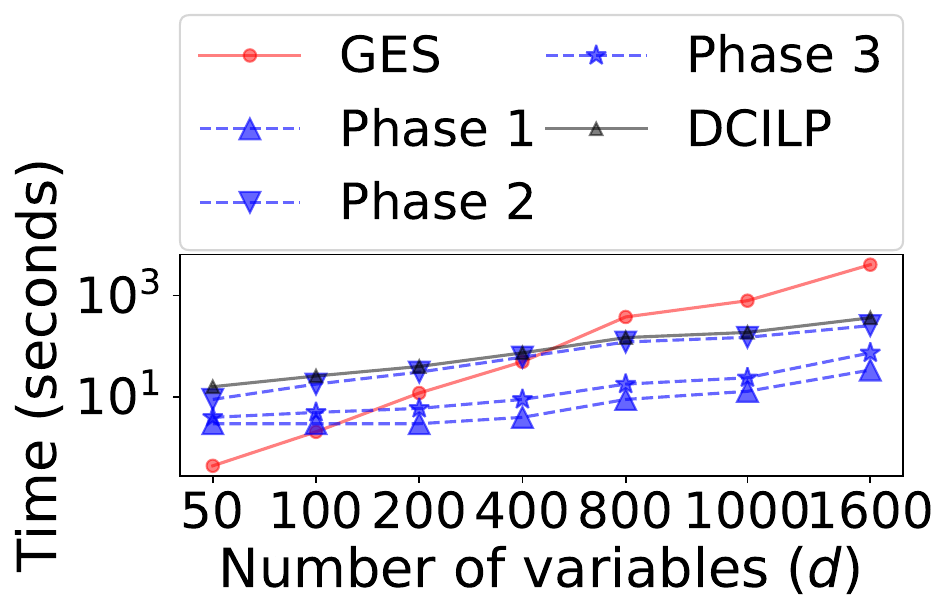}} \hspace{-1.5mm} %
    \subfigure[DCILP-dagma]{\includegraphics[width=.235\textwidth]{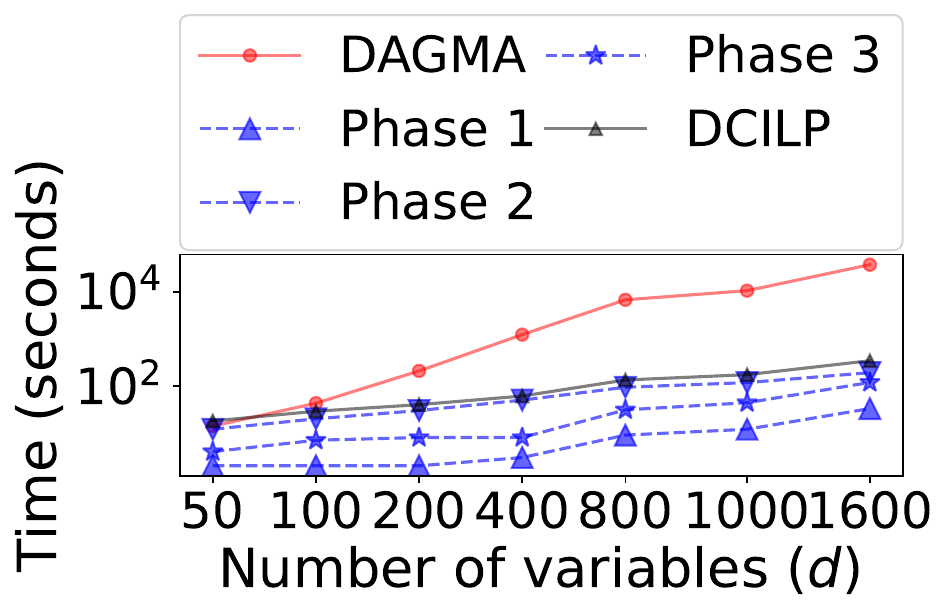}}
    \caption{Computation time of the 3 phases on ER2 data with different graph sizes ($d$). }
    \label{fig:runtimes}
\end{figure}

\begin{table}[htpb]
\small 
\centering
\setlength{\tabcolsep}{5.2pt} %
\begin{tabular}{c|l|ccr}
\hline
       $d$                                  & Algorithm           &  SHD                             & \#edges      & time (sec)         \\ 
\hline 
\multicolumn{1}{c|}{\multirow{6}{*}{200}}   & GOLEM           & 6.0 {$\pm$ 0.58}       & 194            & 686.28     \\ %
\multicolumn{1}{c|}{                    }   & DAS             & 110 {$\pm$  6.60}      & --                & 282.00       \\ %
\multicolumn{1}{c|}{                   }    & GES           	& 177.2 {$\pm$ 10.21}  	& 361            & 4.58        \\ 
\multicolumn{1}{c|}{                   }    & DCILP-ges     	& 53 {$\pm$ 16.17}   	& 243        	   	& 68.33        \\ 
\multicolumn{1}{c|}{                   }    & DAGMA         	& 1.6 {$\pm$ 1.67}   	& 200            & 173.93       \\ 
\multicolumn{1}{c|}{                   }    & DCILP-dagma   	& 15 {$\pm$ 5.57}   	    & 208        	   	& 59.00    \\ 
\hline
\multicolumn{1}{c|}{\multirow{6}{*}{1000}}   & GOLEM          & 31.0 {$\pm$ 8.19}       & 990             & 7885.99      \\ %
\multicolumn{1}{c|}{                     }   & DAS            & 994 {$\pm$ 15.00}      & --                & 5544.00      \\ %
\multicolumn{1}{c|}{                   }    & GES           	& 1521.8 {$\pm$ 51.3} 	& 2470          	& 629.30         \\ 
\multicolumn{1}{c|}{                   }    & DCILP-ges     	& 79 {$\pm$ 5.86}   	    & 1004        	    & 328.00         \\ 
\multicolumn{1}{c|}{                   }    & DAGMA         	& 3.40 {$\pm$ 2.88}   	& 1001          	& 8554.24         \\ 
\multicolumn{1}{c|}{                   }    & DCILP-dagma   	& 46 {$\pm$ 23.59}   	& 1014        	    & 272.67          \\ 
\hline
\end{tabular}
\caption{Results on ER1 data ($n/d = 10$): Comparative performance of DAS, DAGMA, GES, \DG\ and \DD. \DD\ and \DG\ use 400 CPU cores in \phad.}  
\label{tab:er1-withdas}
\end{table}

\begin{figure}[htpb]
    \begin{center}
        \includegraphics[width=.237\textwidth]{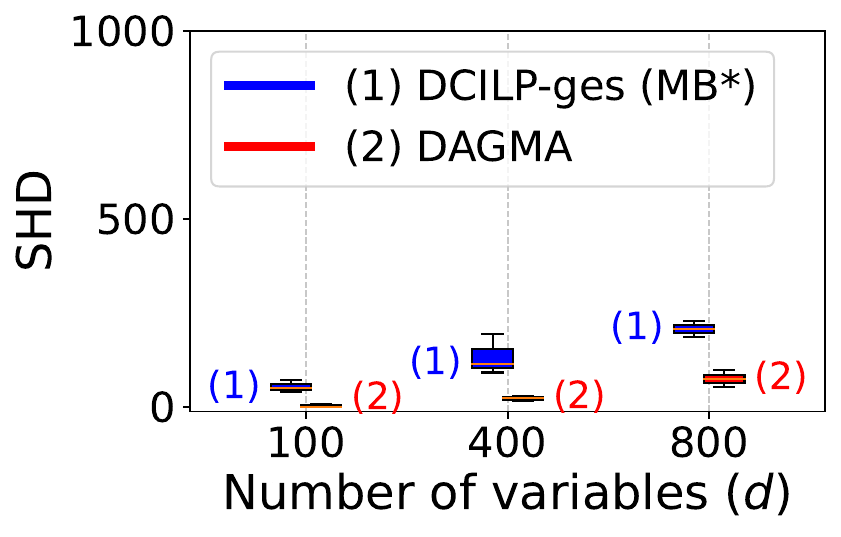} 
        \includegraphics[width=.231\textwidth]{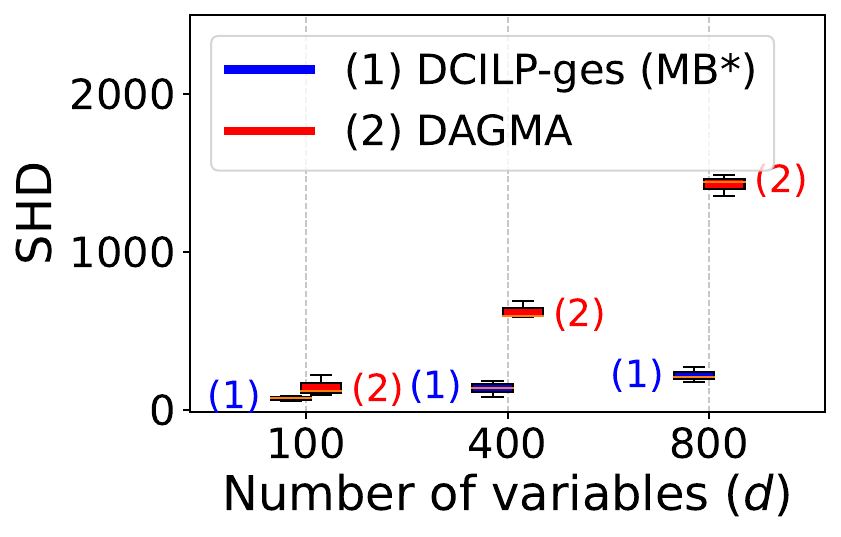} 
    \end{center}
    \caption{Results on ER2 data under the EV (left) and NV (right) settings.} 
    \label{fig:bplot-ev-nv}
\end{figure}

In all settings, the time efficiency gains of \DD\ over DAGMA are even more significant than for \DG\ compared to GES: 
\DD\ achieves an $100\times$ speed up %
over DAGMA for $d \ge 400$. As discussed in \Cref{sec:discuss}, such speedups are largely due to the parallelization of the local learning subproblems in \phad.\footnote{Despite the fact that the parallelization of \phad\
encounters congestion for large $d$ given the available CPU cores.} %

\subsection{Results on Real-World Graph} 

On MUNIN (Fig. \ref{fig:munin-u5}), \DD\ and \DG\ achieve significant reductions in running time compared to DAGMA (with a speed-up of circa 270 times) and compared to GES (with a speed-up of circa 25 times). 
At the same time, their learning accuracy is on par with that of DAGMA, and considerably better than for GES; the SHDs of \XX-dagma and \XX-ges rank first or second  along with DAGMA; the TPRs are similar for DAGMA, GES and \DD, and slightly inferior for \DG\ with uniform noise. 

Similar trends are observed for other $n/d$ ratios and noise types (details in \cite[{Appendix~D.6}]{ArxivUS}).

\begin{figure}[htpb]
\centering
\subfigure[Gaussian noise]{ 
   \includegraphics[width=.211\textwidth]{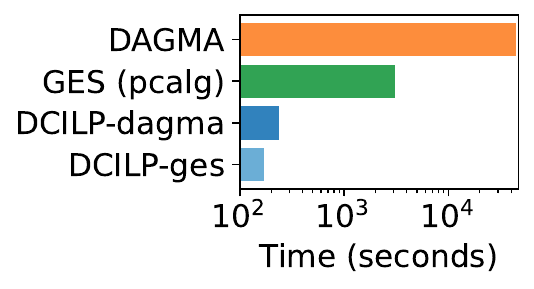} 
   \hspace{-4.2mm}
   \includegraphics[width=.133\textwidth]{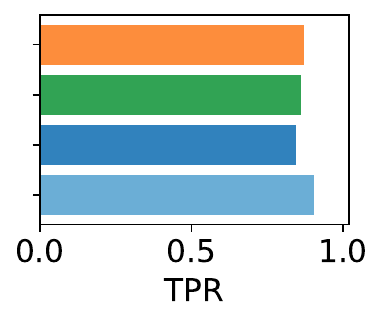}
   \hspace{-3.6mm}
   \includegraphics[width=.133\textwidth]{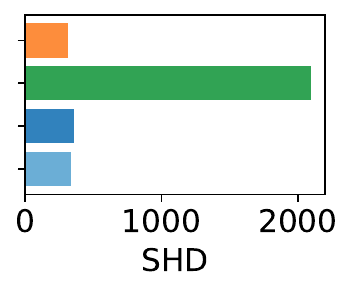}
  }
\subfigure[Uniform noise]{
   \includegraphics[width=.211\textwidth]{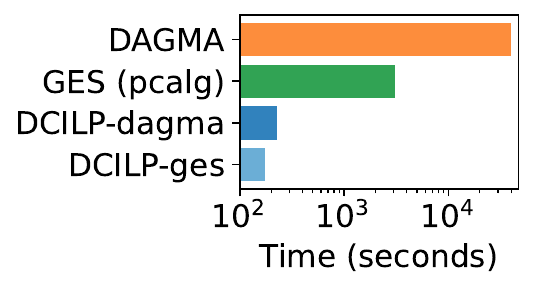} 
   \hspace{-4.2mm}
   \includegraphics[width=.133\textwidth]{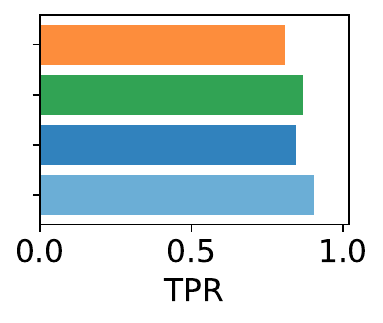}
   \hspace{-3.6mm}
   \includegraphics[width=.133\textwidth]{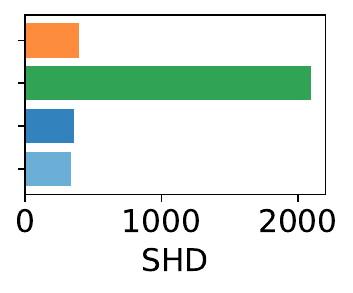}
  }
\caption{Comparative results on MUNIN linear SEM with Gaussian (a) and Uniform (b) noise distributions. The ratio $n/d$ is 5.}
\label{fig:munin-u5}
\end{figure}

\subsection{Impact of ILP in \phat}
\label{ssec:discuss-exp}
\label{parag:exp-lpvsconcat}

The actual effects of the ILP reconciliation process are analyzed by comparing the performance of \DD\ and \DG\ with that of the naive merge $\widehat{B}$~(Eq. \eqref{eq:naive}) of the causal graphs obtained in \PD, %
on the ER2 and SF3 cases. 
As shown in Fig. \ref{fig:exp1-ph3}, the main impact of the ILP-based reconciliation 
is that it significantly reduces FDR, 
at the expense of a slight reduction in TPR, particularly so for high values of $d$ (TPR remaining above 90\%). 

\begin{figure}[htpb]
    \centering
\includegraphics[width=.432\textwidth]{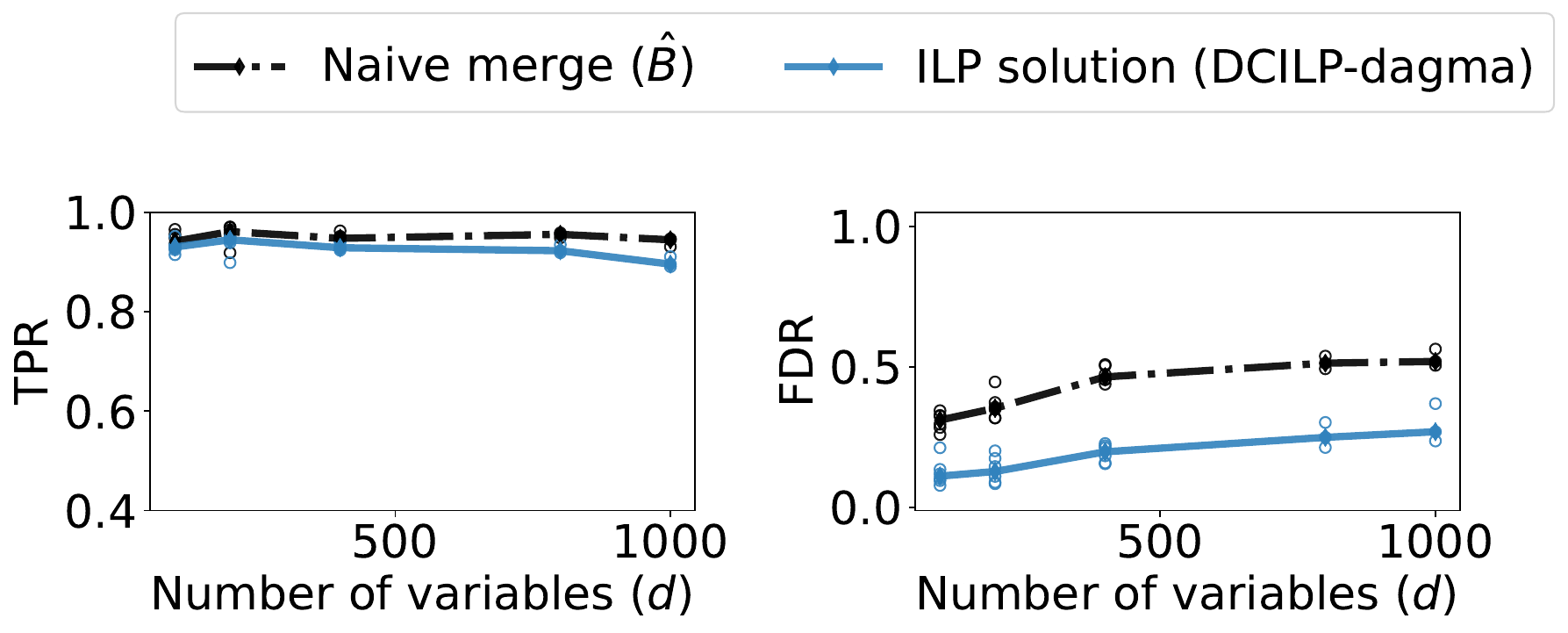}  
\caption{Comparison with the naive merge $\widehat{B}$: \DD\ on SF3 data.}
\label{fig:exp1-ph3}
\end{figure}

\section{Position with Respect to Related Work} 
\label{sec:SoA}

The most related work is the divide-and-conquer strategy proposed by \citep{gu2020learning}, defining the {\em Partition, Estimation and Fusion} (PEF) approach. The partition step proceeds as a hierarchical clustering algorithm. %
The estimation step %
consists of %
estimating the DAG or CPDAG associated with each cluster. 
A first difference between PEF and \XX\ regards the divide phase: in PEF, variables are partitioned while \XX\ considers overlapping Markov blankets. 
Noting that both approaches face the causal insufficiency issue, the redundancy involved in \XX-\phad\ possibly results in a better robustness. The second and most important difference between PEF and \XX\ regards the conquer phase. PEF tackles a combinatorial optimization problem, defining and selecting the best edges among the subgraphs, while \XX\ relies on the ILP resolution to ensure the consistency of the global causal graph built from the local ones. 

A top-down strategy, proposed by \cite{zhang2020learning} and aimed at CPDAG, proceeds by recursively splitting the set of variables in smaller subsets. Inversely, a bottom-up strategy is the graph growing structure learning (GGSL) proposed by \cite{gao2017local} and aimed to Bayesian Networks. In these approaches, a main goal is to reduce the number of conditional independence tests needed to elicit the overall causal structure. 
Another related work, though to a lesser extent, is the recursive variable elimination method named MARVEL proposed by~\cite{mokhtarian2021recursive}. 
Like \XX, the MARVEL algorithm relies on the identification of Markov blankets, using e.g. Grow-Shrink (GS)~\citep{margaritis1999bayesian}, IAMB~\citep{tsamardinos2003algorithms}, %
or precision matrix-based methods (\eg, \citep{loh2014high}). %
The main issue in practice is that the recursive process offers no room to recover from potential early mistakes.

\cite{jaakkola2010learning} tackle the problem of learning Bayesian network structures, where the acyclicity constraint is handled using linear programming relaxations. 
\cite{cussens2017bayesian,cussens2023branch} proposed integer linear programming methods for Bayesian network learning, by which the IP problem considers binary variables $x_{i\leftarrow J}$ defined on the ensemble of subsets as indicators of whether a subset $J$ is the parent set of $X_i$. This problem setting is proposed for exact learning of Bayesian networks (DAGs) based on a decomposable score function; in particular, it does not make use of Markov blankets, and %
involves much more decision variables than DCILP. %
\section{Conclusion}
\label{sec:conclu}
The divide-and-conquer \XX\ framework is based on the natural decomposition of the desired causal graph defined by the Markov blankets $\mbo(X_i)$ associated with each variable $X_i$. The main contribution of the paper is to demonstrate that the causal relationships involving $X_i$ $-$ learned in parallel by focusing on $\mbo(X_i)$ $-$ can be reconciled using an ILP formulation, and that this reconciliation can be efficiently handled by an ILP solver.
This three-phase framework accommodates a variety of choices for each phase, from identifying the Markov blankets in \phao\ to causal learning in \PD\ and selecting the ILP solver in \PT. The experimental results effectively demonstrate the framework's merits, particularly so on the large-scale real-world MUNIN data, 
achieving a computational time reduced by two orders of magnitude with on-par SHD compared to DAGMA (respectively, one order of magnitude in computational time with a significant gain in SHD compared to GES). 

A main limitation of the approach comes from the learning accuracy of the Markov blankets and the local causal relations in \phao\ and \PD, which require a sufficient samples-to-variables ratio.
 
One avenue for future research is to consider the three phases in a 
more integrated manner, %
\eg, taking into account the strength of the causal edges identified in \PD\ within the ILP objective in \PT\ (as opposed to only taking into account the existence of such edges). 

Another perspective is to leverage the multiple solutions discovered by the ILP solver in \PT. For instance, the causal relations retained in several ILP solutions can be used to define the backbone of the desired causal graph, thus extending \XX\ in the spirit of ensemble learning. A longer-term perspective is to approximate well-studied causal scores \cite{loh2014high} using bilinear functions, in order to extend \PT, from integer linear programming to quadratic programming.

\appendix
\section*{Acknowledgments} 
We thank the anonymous reviewers for their helpful comments and suggestions. The first author was funded by a grant from Fujitsu, and the work was also supported by the CAUSALI-T-AI project of PEPR-IA (ANR-23-PEIA-0007 CAUSALI-T-AI).



\newpage
\onecolumn
\section*{Appendix}\label{sec:sm}

\section{Organization of the appendix}

The appendix is organized as follows:
\begin{itemize}
    \item Proof of \Cref{def:mergeco} 
    \item Detail of the algorithms (\phao\ to \phat)
    \item Complementary experiments 
\end{itemize}

\section{Proof of \Cref{def:mergeco}}
\label{sec-app:prf}

For $\widehat{B}^{(i)}$ ($i\in [d]$) obtained in \phad\ of \Cref{alg:dc-ges}, the edge set of its graph is restricted within %
\begin{align*}%
    C_{i}= \{(i,k): k \in \mbo(X_i)\}\cup  \{(k,i): k \in \mbo(X_i)\}. 
\end{align*}
Hence the intersection of the support graphs of $\widehat{B}^{(i)}$ and $\widehat{B}^{(j)}$ for any pair $i\neq j$ is included in $C_i\cap C_j=\{(i,j), (j,i)\}$.  
Therefore, by \Cref{defi:mergeco},  
$\widehat{B}^{(i)}$ and $\widehat{B}^{(j)}$ result in a merge conflict if and only if    
\begin{align*}
    \widehat{B}^{(i)}_{ij} \neq \widehat{B}^{(j)}_{ij}  \text{~or~} 
\widehat{B}^{(i)}_{ji} \neq \widehat{B}^{(j)}_{ji}. 
\end{align*}

As a consequence, all merge conflicts can be classified into the following three types, according to the number of nonzeros in the quadruplet 
$  Q_{ij}=\{\widehat{B}^{(i)}_{ij}, \widehat{B}^{(j)}_{ij}, \widehat{B}^{(i)}_{ji}, \widehat{B}^{(j)}_{ji}\}$: 

\begin{table}[h]
\small 
\centering
\begin{tabular}{l|cc|cc|c|c}
\hline
Type  & $\hat{B}^{(i)}_{ij}$  & $\hat{B}^{(i)}_{ji}$ & $\hat{B}^{(j)}_{ij}$  & $\hat{B}^{(j)}_{ji}$ & Graphical model    & Characteristics \\ 
\hline
\multicolumn{1}{c|}{\multirow{4}{*}{\it (3) Undirected}} 	& * 	& 0   & *    	& *   	&  $(i\rightarrow j)$ vs $(i\leftrightarrow j)$  & \multicolumn{1}{c}{\multirow{4}{*}{$\nnz(Q_{ij})=3$}}   \\ 
\multicolumn{1}{c|}{                    } 	                & 0 	& *   & *    	& *    	&  $(i\leftarrow j)$ vs $(i\leftrightarrow j)$   &  \multicolumn{1}{c}{                    }               \\ 
\multicolumn{1}{c|}{                    } 	                & \vdots & \vdots & \vdots    	& \vdots    	&  \vdots                        &  \multicolumn{1}{c}{                    }                \\ 
\multicolumn{1}{c|}{                    } 	                & * 	& *   & 0    	& *    	&  $(i\leftrightarrow j)$ vs $(i\leftarrow j)$   &  \multicolumn{1}{c}{                    }                \\ 
\hline
\multicolumn{1}{c|}{\multirow{4}{*}{\it (2) Acute}} 	& * 	& 0   & 0    	& *   	&  $(i\rightarrow j)$ vs $(i\leftarrow j)$  & \multicolumn{1}{c}{\multirow{4}{*}{$\nnz(Q_{ij})=2$ }}   \\ 
\multicolumn{1}{c|}{                    } 	            &  0 	&  *   & *    	& 0    	&  $(i\leftarrow j)$ vs $(i\rightarrow j)$   &  \multicolumn{1}{c}{                    }               \\ 
\multicolumn{1}{c|}{                    } 	            &  * 	&  *   & 0    	& 0    	&  \redl{$(i\leftrightarrow j)$ vs $(i\ind j)$}   &  \multicolumn{1}{c}{                    }               \\ 
\multicolumn{1}{c|}{                    } 	            &  0 	&  0   & *    	& *    	&  \redl{$(i\ind j)$ vs $(i\leftrightarrow j)$}   &  \multicolumn{1}{c}{                    }               \\ 
\hline
\multicolumn{1}{c|}{\multirow{4}{*}{\it (1) Addition}} 	& * 	& 0   & 0    	& 0   	&  $(i\rightarrow j)$ vs $(i\ind j)$  & \multicolumn{1}{c}{\multirow{4}{*}{$\nnz(Q_{ij})=1$ }}   \\ 
\multicolumn{1}{c|}{                    } 	            & 0 	& *   & 0    	& 0    	&  $(i\leftarrow j)$ vs $(i\ind j)$   &  \multicolumn{1}{c}{                    }               \\ 
\multicolumn{1}{c|}{                    } 	            & \vdots & \vdots & \vdots    	& \vdots    	&  \vdots                        &  \multicolumn{1}{c}{                    }                \\ 
\multicolumn{1}{c|}{                    } 	            & 0 	& 0   & 0    	& *    	&  $(i\ind j)$ vs $(i\leftarrow j)$  &  \multicolumn{1}{c}{                    }                \\ 
\hline
\end{tabular}
\caption{Classification of all merge conflicts from a pair of local results. The symbol * indicates a nonzero number.} \label{tab:dnc-ges-mf}
\end{table}

$\square$

\section{Algorithms}
\label{sec-app:algo}
This section describes the algorithms involved in \phao, \PD\ and \phat\ of \XX. 

\subsection{\phao: an empirical inverse covariance estimator} %
\label{ssec:app-sel}

\begin{algorithm}[H] 
    \caption{Empirical inverse covariance estimator}\label{alg:ice-emp}
    \begin{algorithmic}[1]
        \REQUIRE{Data matrix $X\in\reals^{n\times d}$, parameter $\lambda_1\in (0,1)$} 
        \ENSURE{$\widehat{\Theta}_{\lambda_1}\in\dom$} 
\STATE Compute empirical covariance and its inverse: 
\begin{align}
    \label{eq:ice-emp}
    \widehat{C} = \frac{1}{n} \trs[(X-\bar{X})] (X-\bar{X})\quad \text{and} \quad
    \widehat{\Theta} = \widehat{C}^{\dagger}, %
\end{align}
where $\widehat{C}^{\dagger}$ denotes the pseudo-inverse of $\widehat{C}$. 
\STATE Element-wise thresholding on off-diagonal entries: for any $i\neq j$ 
\begin{align}
    \label{eq:theta-soft-th}
    (\widehat{\Theta}_{\lambda_1})_{ij} = \twopartdef{ 0 }{ |\widehat{\Theta}_{ij}| \leq \lambda_1 \max_{i\neq j}\{|\widehat{\Theta}_{ij}|\} }{ \widehat{\Theta}_{ij}}{\text{otherwise.}} 
\end{align}
\end{algorithmic}
\end{algorithm}

In the presented implementation of \XX, we use a basic empirical inverse covariance estimator  for the inference of Markov blankets (\Cref{alg:ice-emp}). 
In the computation of Eq.~\eqref{eq:ice-emp} (below), the pseudo-inverse $\widehat{C}$ coincides with the inverse of $\widehat{C}$ when $\widehat{C}$ is positive definite (which is the case for $n \ge d$). %
In the computation of matrix $\widehat{\Theta}_{\lambda_1}$~\eqref{eq:theta-soft-th}, the off-diagonal entries of the input matrix $\widehat{\Theta}$ are canceled out if their absolute value is smaller than $\lambda_1 \max_{i\neq j}\{|\widehat{\Theta}_{ij}|\}$  
for a parameter $\lambda_1 \in (0,1)$.

\paragraph{Selection of $\lambda_1$ for \Cref{alg:ice-emp}.}
\label{ssec-app:sel-ice-emp}
The parameter $\lambda_1$ is selected in the spirit of the GraphLasso objective function \cite{Friedman2007}\footnote{\cite{Friedman2007}: Friedman J, Hastie T, Tibshirani R. Sparse inverse covariance estimation with the graphical Lasso. Biostatistics. 2008 Jul;9(3):432-41.}; the selected value is the maximum of the following criterion function 
\begin{align} \label{eq:crit-c1}
C(\lambda_1) := \trace(\widehat{C}\widehat{\Theta}_{\lambda_1}) - \log\det(\widehat{\Theta}_{\lambda_1})
\end{align}
on a 20-sample grid in $[\lambda_1^{\min}, \lambda_1^{\max}]$. 

The total time for selecting $\lambda_1$ with \Cref{alg:ice-emp} 
corresponds to the computational time of \phao\ in the benchmark discussed in
\Cref{sec:exp}. 

Based on the above setting, \phao\ involves running \Cref{alg:ice-emp} for each candidate value of $\lambda_1$ in a given segment $[\lambda_1^{\min}, \lambda_1^{\max}]$. The following is an example of how to set up an appropriate value for $\lambda^{\max}_1$, assuming  
that the sought causal structures have an average degree $1\leq \degr\leq 4$: the target sparsity of $\widehat{\Theta}_{\lambda_1}$ by \Cref{alg:ice-emp} is bounded by $\bar{\rho}_{\degr}=\max(\frac{\degr}{d})\approx 2.0\%$ for graphs with $d\geq 200$ nodes. 
This gives us an approximate target percentile of around $98\%$, \ie, retaining the $2\%$ edges with max. absolute weight.
In other words, the maximal value $\lambda_1^{\max}$ of the grid search area is set as 
$\lambda_1^{\max} :=\frac{|\widehat{\Theta}_{\mathrm{off}}(\ttau)|}{\maxn[\widehat{\Theta}_{\mathrm{off}}]}$, 
where $\ttau$ refers to the index of the $98$-th percentile in $\{|\widehat{\Theta}_{\mathrm{off}}|\}$. %
For the experiments with ER2 graphs in \Cref{sec:exp}, the estimated $\lambda_1^{\max}$ is $6.10^{-1}$. Hence, the search grid of $\lambda_1$ is set up as $n_{I_1}=20$ equidistant values on $I_1 = [10^{-2}, 6.10^{-1}]$. 
More details about the segment settings are given case by case in \Cref{ssec-app:params}.

\begin{figure}[H]
  \centering
     \subfigure[Criterion Eq.~\eqref{eq:crit-c1} used]
   {\includegraphics[width=.24\textwidth]{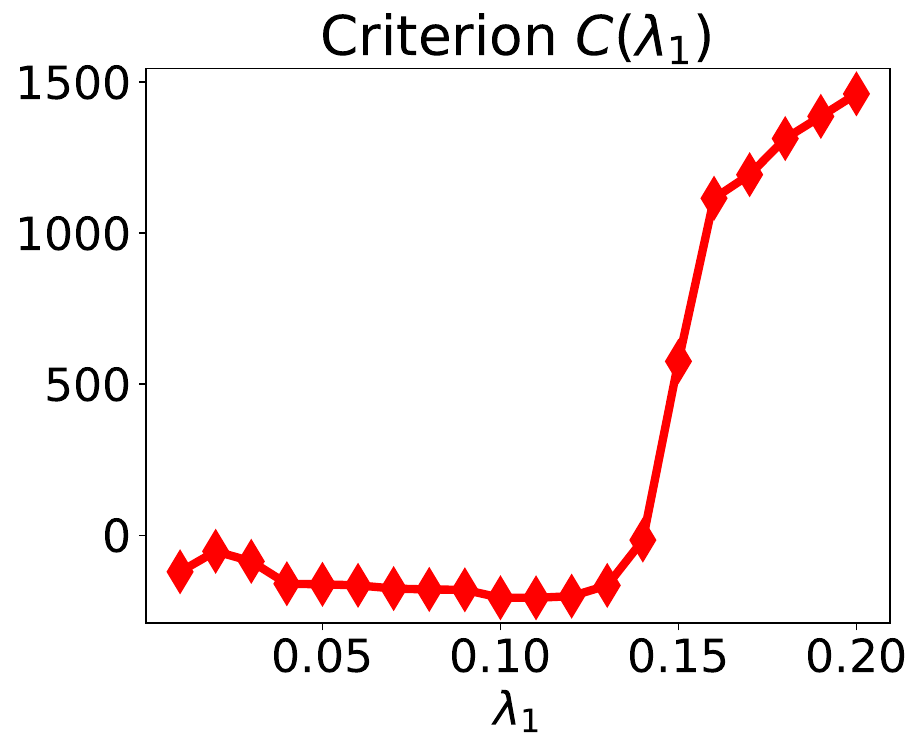}}
  \hspace{1cm} 
  \subfigure[Distance$(\widehat{\Theta}_{\lambda_1},\Theta^{\star})$ with ground-truth knowledge]
  {\includegraphics[width=.24\textwidth]{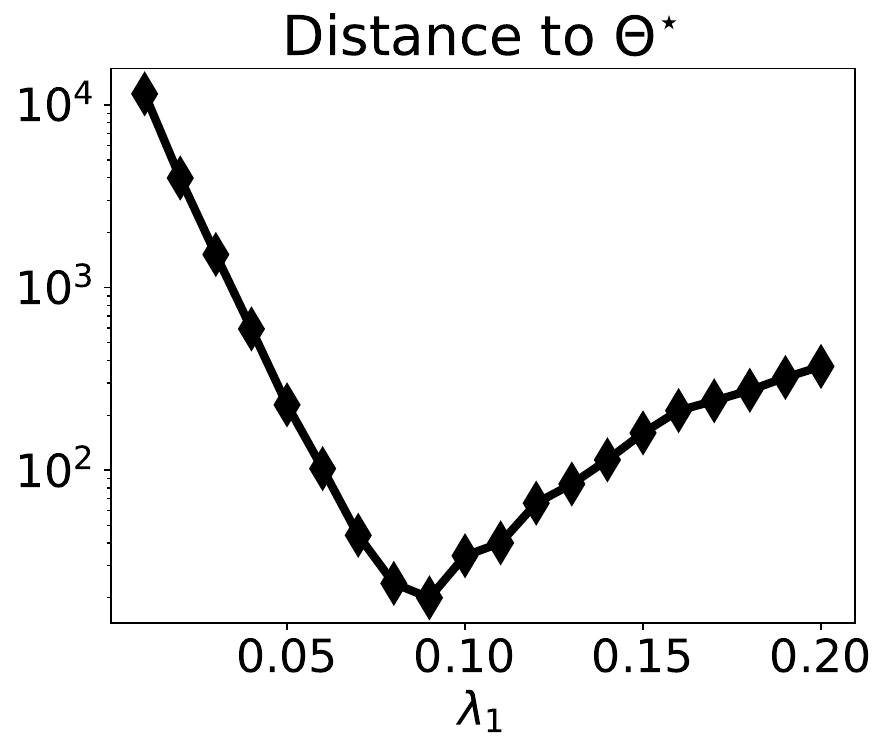}}   
    \caption{Selecting parameter $\lambda_1$ in \Cref{alg:ice-emp} in the
    case of linear Gaussian SEM on ER2 ($d=200$ nodes): (a)
    Selection of $\widehat{\Theta}_{\lambda_1}$ based on $C(\lambda_1)$ %
    (Eq.~\eqref{eq:crit-c1}); 
    (b) Distance of $\widehat{\Theta}_{\lambda_1}$ to the ground truth $\Theta^\star$. 
    }
\label{fig:app-sel-lam1}
\end{figure}

\Cref{fig:app-sel-lam1} shows the criterion value compared to the Hamming distances with the oracle precision matrix $\Theta^{\star}$. %
We observe that the selection criterion with $\argmin_{I_{1}} C(\lambda_1)$ yields a value
that is rather close to the optimal value in terms of distance of $\widehat{\Theta}_{\lambda_1}$ to the oracle precision matrix $\Theta^{\star}$. 

\subsection{\phat: the ILP}
\label{ssec-app:ilp-size}

The ILP formulation consists of maximizing the inner product $\braket{\widehat{B}, B}$ subject to 
constraints \eqref{eq:mb}--\eqref{eq:v-str4}, %
along with the sparsity constraint \eqref{eq:bb} and the 2-cycle exclusion constraint \eqref{eq:2cyc}: 

\begin{align}
     & \underset{B\in\mathbb{B}^{d\times d}, S\in\mathbb{S}^{d\times d},
        V\in\mathbb{B}^{d\times d\times d}}{\text{maximize}} \braket{\widehat{B}, B} %
     \quad \text{subject to } \label{eq-app:fobj-lp-cp} \\ 
     & \hspace{9mm} { B_{ij}=B_{ji} =0  \hspace{36mm} 
     \text{~if ~} (\widehat{B}_{ij}=\widehat{B}_{ji} = 0) } \label{eq-app:sp1-cp} \\ 
     & \hspace{9mm} S_{ij} =S_{ji} =0  \hspace{38mm} \text{~if~} X_i \notin \mbo(X_j) \label{eq:sp} \\ %
    & \hspace{9mm} V_{ijk} =0  \hspace{46mm} \text{~if~} \widehat{B}_{ik}=0 \text{~or~} \widehat{B}_{jk} = 0 \label{eq-app:sp3-cp} \\ 
    & \hspace{9mm} B_{ij} + B_{ji} \leq 1, \quad  B_{ij} + B_{ji} + S_{ij} \geq 1  \quad \text{if~}  X_i\in\mbo(X_j) \label{eq-app:2cycle-exc-cp}\\ 
    & \hspace{9mm} V_{ijk} \leq B_{ik}, ~ V_{ijk} \leq B_{jk}, ~ V_{ijk} \leq S_{ij}  \quad \text{and}  \label{eq-app:v1-cp}\\ 
    & \hspace{9mm} B_{ik} + B_{jk} \leq 1 + V_{ijk}, ~ S_{ij} \leq \sum_{k} V_{ijk} \quad \text{if~} \widehat{B}_{ik} \neq 0, \widehat{B}_{jk}\neq 0.   \label{eq-app:v2-cp} 
\end{align}
In view of \Cref{def:mergeco}, the following rules can be used to attribute edge weights to the concatenation $\widehat{B}=\sum_{i}\widehat{B}^{(i)}$~(Eq. \ref{eq:naive}), where $\widehat{B}^{(i)}$ is the $(0,1)$-valued matrix obtained in \phad.  

       \begin{table*}[h]
           \centering 
        \begin{tabular}{cc|c|c|c|c}
            \hline
            \multirow{2}{*}{$(\hat{B}_{ij}, \hat{B}_{ji})$} & \multirow{2}{*}{Merge status} &  \multicolumn{4}{c}{\multirow{1}{*}{$(\widehat{B}_{ij}, \widehat{B}_{ij})$}}  \\ 
                                                            &                               &  LP1 & LP2 & LP3 & LP4                    \\ 
            \hline
            (0,0)   & \multirow{4}{*}{no conflict}         & \multirow{4}{*}{$\bm{1}(\cdot,\cdot)$} & \multirow{4}{*}{$\text{id}(\cdot,\cdot)$} & \multirow{4}{*}{$\text{id}(\cdot,\cdot)$}  &  \multirow{4}{*}{$\bm{1}(\cdot,\cdot)$}  \\ 
            (2,0)   &                     &                                                         &                                           &                                            &                                      \\  
            (0,2)   &                     &                                                         &                                           &                                            &                                      \\  
            (2,2)   &                     &                                                         &                                           &                                            &                                       \\ 
            \midrule 
            (0,1)   & \multirow{2}{*}{Type 1 (Addition)}  &  \multirow{2}{*}{$\bm{1}(\cdot,\cdot)$} & \multirow{2}{*}{$\text{id}(\cdot,\cdot)$} & \multirow{2}{*}{$\text{id}(\cdot,\cdot)$} &       (0, $\frac{1}{2}$)                             \\ 
            (1,0)   &                                     &                                         &                                           &                                           &       ($\frac{1}{2}$, 0)                             \\ 
            \midrule 
            (1,1)   & Type 2 (Acute)                      &  \multirow{1}{*}{$\bm{1}(\cdot,\cdot)$} & \multirow{1}{*}{$\text{id}(\cdot,\cdot)$} & \multirow{1}{*}{$\text{id}(\cdot,\cdot)$}  &     ($\frac{1}{2}$, $\frac{1}{2}$)                  \\ 
            \midrule 
            (2,1)   & \multirow{2}{*}{Type 3 (Dir/Undir)}  &  \multirow{2}{*}{$\bm{1}(\cdot,\cdot)$} &\multirow{2}{*}{$\text{id}(\cdot,\cdot)$} &  (2, 0)      &  ($\frac{2}{3}$, 0)                              \\  
            (1,2)   &                                      &                                         &                                          &  (0, 2)      &  (0, $\frac{2}{3}$)                              \\  
            \midrule
        \end{tabular}
\caption{Concatenation weights depending on the merge status. The original concatenation states that $\widehat{B}_{ij}= \widehat{B}^{(i)}_{ij} + \widehat{B}^{(j)}_{ij}$ for all $i,j\in [d]$, $j\neq i$.} \label{tab-app:grb-ycode} 
       \end{table*}

\subsection{Parallelization in \phad}

In the presented experiments, the parallelization of \XX' \phad\ (\Cref{alg:dc-ges}, line 2) is limited on a maximum of 400~CPU cores. This is ensured by the SLURM scheduler via the command ``\texttt{sbatch --ntasks=}$N_{\max}$'', where $N_{\max} = \min(2d, 400)$ for any problem with $d$ variables. According to how the subproblems' sizes differ, depending on the Markov blanket sizes, we have two types of parallelization schemes as follows.

\paragraph{\phad: parallelization with equal number of CPU cores for each subproblem.}
In cases where the Markov blanket sizes are fairly %
homogeneous, 
it is natural to distribute the subproblems of DCILP' \phad\ onto a fixed number (one or two) of CPU cores. 
The procedure is detailed in \Cref{alg:para0}.

\begin{algorithm}[H]
    \caption{(\phad) Parallelization with equal number of CPU cores}\label{alg:para0}
    \begin{algorithmic}[1]
        \REQUIRE{Data matrix $X\in\reals^{n\times d}$, Markov blankets $\mbo(X_i)$ for $i=1,\dots, d$, $N_{\max} = 400$ } 
        \ENSURE{Subproblem solutions $B^{(i)}$} 
\FOR{$i=1,\dots, d$ {\bf in parallel}} 
\WHILE{CPU cores available (among $N_{\max}$)} %
\STATE{Launch the $i$-th subproblem on 2 CPU cores: }
\STATE{ \hspace{1cm} $A^{(i)} \leftarrow$ Causal discovery on $\si[i]:=\mbo(X_i)\cup\{X_i\}$ \hfill \# Using DAGMA or GES (pcalg)} 
\STATE{ \hspace{1cm} $\widehat{B}^{(i)}_{j,k} \leftarrow A^{(i)}_{j,k}$ if $j=i$ or $k=i$, \text{~ and~} 0 otherwise 
}
\STATE{(Release CPU cores when finished) \hfill \#  managed by SLURM }
\ENDWHILE
\ENDFOR
    \end{algorithmic}
\end{algorithm}

The above scheme is used in all tests on ER1 and ER2 datasets. 

\paragraph{\phad: parallelization with non-uniform number of CPUs (in presence of large Markov blankets). }

In cases where some Markov blankets are much larger than others, 
we use a non-uniform distribution of CPU cores to the subproblems of DCILP' \phad, %
depending on the size of the MBs. The procedure is detailed in \Cref{alg:para-ncpus}. 

The execution of the parallel for-loop in \Cref{alg:para-ncpus} is limited by the same predefined CPU budget ``\texttt{--ntasks=}$N_{\max}$'', which is ensured and managed automatically by SLURM.
Note that when the number of variables $d \geq N_{\max}$, it is certain that congestion happens, i.e., some tasks will need to wait in the queue for the earlier tasks to finish. In
such cases, the subproblems are ordered by decreasing size of the Markov blankets, meant to put first the largest (longest) tasks. 
Furthermore, the top $\rho d$ largest tasks are allocated with 4 CPU cores (instead of two), where $\rho$ is set to 3\% or 8\% in the experiments. 

\begin{algorithm}[H] 
    \caption{(\phad) Parallelization with non-uniform number of CPUs (in case of large Markov blankets) }\label{alg:para-ncpus}
    \begin{algorithmic}[1]
        \REQUIRE{Data matrix $X\in\reals^{n\times d}$, Markov blankets $\mbo(X_i)$ for $i=1,\dots, d$, $N_{\max} = 400$, percentage $\rho \in (0,1)$ } 
        \ENSURE{Subproblem solutions $B^{(i)}$} 
\STATE Order variables indices by decreasing size of Markov blankets
                $|\mbo(X_{k_1})| \geq \dots \geq |\mbo(X_{k_d})|$. 
\FOR{$i=1,\dots, d$ {\bf in parallel}} 
\WHILE{CPU cores (among $N_{\max}$) available} %
\IF{ $i\leq \rho d$  }
\STATE{Launch the $k_i$-th subproblem on 4 CPU cores:} %
\STATE{ \hspace{1cm} $\widehat{B}^{(k_i)} \leftarrow$ \textsc{BatchLocalLearn}($\mbo(X_{k_i}) \cup \{X_{k_i}\}$) \hfill\# see \Cref{alg:bll-dagma}}
\STATE{(Release CPU cores when finished) \hfill \# managed by SLURM }
\ELSE{}
\STATE{Launch the $k_i$-th subproblem on 2 CPU cores: }
\STATE{ \hspace{1cm} $A^{(k_i)} \leftarrow$ Causal discovery on $\si[k_i]:=\mbo(X_{k_i})\cup\{X_{k_i}\}$ \hfill \# Using DAGMA or GES (pcalg)} 
\STATE{ \hspace{1cm} $\widehat{B}^{({k_i})}_{j,\ell} \leftarrow A^{(k_i)}_{j,\ell}$ if $j=k_i$ or $\ell=k_i$, \text{~ and~} 0 otherwise 
}
\STATE{(Release CPU cores when finished) \hfill \# managed by SLURM }
\ENDIF
\ENDWHILE
\ENDFOR
    \end{algorithmic}
\end{algorithm}

This scheme is used in all experiments with SF3 and MUNIN datasets.

        \begin{algorithm}[H]
            \caption{(BatchLocalLearn) Local learning in batches for the causal subproblem} 
\label{alg:bll-dagma}
            \begin{algorithmic}[1]
                \REQUIRE{Dataset $X\in\reals^{n\times d}$, target node $T$, Markov blanket $\mbo(T)$, batch size proportion $\rho \in (0,1)$  }
                \ENSURE{Adjacency matrix $B$ encoding the parent set and children set of $T$ }
                \STATE Batch size $s := \lceil {\rho} |\mbo(T)| \rceil$  
                \STATE Initialization: $\pcs_T \leftarrow \emptyset$, \quad $O \leftarrow \mbo(T)$   
                \WHILE {$O$ is nonempty} 
                \STATE \bleu{Choose a small batch} of $s$ nodes: $S \subset O$ if $|O|\geq s$, otherwise $S\leftarrow O$ 
                \STATE Update subsets $O \leftarrow O \backslash S$, \quad 
                        $ Z \leftarrow S \cup \pcs_T \cup \{X\}$ 
                    \STATE Update $\pcs_T$: 
                        \begin{align} 
                            & \grp' \leftarrow \textsc{DAGMA}(Z, X_{|Z}) \text{~ ~or~ ~} \textsc{GES}(Z, X_{|Z}) \label{eq:bll-dagma} \\
                            & \pcs_T, \pss_T, \chs_T  \leftarrow \textsc{FindPC} (\grp', T). \label{eq:findpc}
                        \end{align}
                \ENDWHILE
                \STATE Encode the parent set $\pss_T$ and children set $\chs_T$ of $T$ into the adjacency matrix $B$ 
            \end{algorithmic}
        \end{algorithm}

In line 6 of \Cref{alg:para-ncpus}, \textsc{BatchLocalLearn} denotes the alternative subroutine for solving the subproblem (learning the causes and effects of the target variable). %
This subroutine, detailed in \Cref{alg:bll-dagma}, is an adaptation of the LocalLearn algorithm from \cite{gao2017local}. \textsc{BatchLocalLearn} differs with LocalLearn in: (i) it takes in small 
subsets %
of nodes instead of one node during the iterative update process; and (ii) only the parent/children sets are updated iteratively while LocalLearn had an additional task of updating the Markov blanket. 

In \Cref{alg:bll-dagma}, ($\pss_T$, $\chs_T$, $\pcs_T$) denote the parent set (the set of parent nodes), the children set, and the union of parent-and-children nodes of a target node $T$, respectively. In line 6 (Eq.~\eqref{eq:findpc}) of \Cref{alg:bll-dagma}, given the adjacency matrix of a graph $\grp'$, \textsc{FindPC} denotes the operation of getting $\pss_T$ (encoded by the nonzeros in $T$'s column), $\chs_T$ (nonzeros in $T$'s row), and their union $\pcs_T$. 
In line 6 (Eq.~\eqref{eq:bll-dagma}), one of the two local learners (either DAGMA or GES) is used, according to the choice in \Cref{alg:para-ncpus} for \phad. %

\section{Experiments}
\label{sec-app:exp}

\subsection{Evaluation metrics} 
\label{ssec-app:metrics}

Let T (respectively F) denote the number of true edges (resp. non-edges) in the  ground truth graph $\cal G$. Let P the number of edges in the estimated graph, R the number of reversed edges, TP and FP the number of true positives and false positives. %
Let E be the number of edges in the estimated graph that do not appear in $\cal G$ and $M$ the number of edges in $\cal G$ that do not appear in the estimated graph. 
With these notations, the usual performance indicators for causal graph learning are defined as follows:
\begin{align*}
    \text{TPR} & = \text{TP/T                      \quad (higher is better),} \\ 
    \text{FDR} & = \text{(R + FP)/P           \quad  (lower is better),}  \\ 
    \text{FPR} & = \text{(R + FP)/F            \quad  (lower is better),}  \\ 
    \text{SHD} & = \text{E +M + R              \quad (lower is better).} 
\end{align*}

\subsection{Computational environment}
\label{ssec-app:comp-src}

The computation of \XX' \phad\ is distributed on $N_{\max}=\min(2d, 400)$ CPU cores, as specified in the description of \Cref{alg:para0} and \Cref{alg:para-ncpus}. The cluster for this parallel computation has three types of CPUs: 
\begin{itemize}
    \item Intel Xeon E5-2620 v4 8 cores / 16 threads @ 2.40 GHz (Haswell)
    \item Intel Xeon Gold 5120 14 cores / 28 threads @ 2.2GHz (Skylake) 
    \item Intel Xeon Gold 6148 20 cores / 40 threads @ 2.4 GHz (Skylake)
\end{itemize}
The computation of \phao\ and \phat\ are conducted on two and four CPU cores from the same cluster.

DAGMA is run on one CPU core from a 28-thread CPU of Intel(R) Xeon(R) Gold 5120. 
GES (pcalg) is run on one CPU core from a 4-thread Intel(R) Xeon(R) Gold 6252. 
GOLEM is run on a Tesla V100S-PCIE-32GB GPU (with 8-thread AMD EPYC 7352).

\subsection{Learning parameters and SLURM parameters}
\label{ssec-app:params}

The parameters of \XX\ are as follows:
\begin{itemize}
    \item[(a)] For \phao: the range  $[\lambda_1^{\min}, \lambda_1^{\max}]$ for the grid search (with 20 equidistant values therein) using the inverse covariance estimator (\Cref{alg:ice-emp}) 
    \item[(b)] For \phad\ of \XX-dagma: parameter $\lambda_2$ of DAGMA 
    \item[(c)] For \phad: choice of parallelization between \Cref{alg:para0} and \Cref{alg:para-ncpus}. 
\end{itemize}

\paragraph{In the experiments on ER1, ER2 datasets:} 
\begin{itemize}
    \item[(a)]  $\lambda_1^{\min} = 0.05$, $\lambda_1^{\max} = 0.3$  
    \item[(b)]  $\lambda_2 = 0.01$ 
    \item[(c)] Parallelization of \phad\ is conducted with \Cref{alg:para0}. 
\end{itemize}

\paragraph{In the experiments on SF3 datasets:} 
\begin{itemize}
    \item[(a)]  $\lambda_1^{\min} = 0.02$, $\lambda_1^{\min} = 0.1$ for $d\geq 800$; and $\lambda_1^{\min} = 0.003$, $\lambda_1^{\max} = 0.1$ for $d< 800$.
    \item[(b)]  $\lambda_2 = 0.07$  
    \item[(c)] Parallelization of \phad\ is conducted with \Cref{alg:para-ncpus}
\end{itemize}

\paragraph{In the experiments on MUNIN datasets:} 
\begin{itemize}
    \item[(a)]  $\lambda_1^{\min} = 0.015$, $\lambda_1^{\min} = 0.07$ 
    \item[(b)]  $\lambda_2 = 0.001$. Disclaimer: we observe in preliminary tests that the performance of DAGMA is fairly stable with different choices of $\lambda_2$. The value set up here is lower than in the other experiments to have good ``recalls'' (for the local solutions to have TPRs as high as possible); it is possible to have similar performance with higher $\lambda_2$ values. 
    \item[(c)] Parallelization of \phad\ is conducted with \Cref{alg:para-ncpus} 
\end{itemize}

\subsection{Running time of \XX}

The running time benchmark of \XX-GES in the causal discovery tasks on ER1 and ER2 data is given in \Cref{fig-app:exp1-dcges-wt}.  

The running time benchmark of \XX-dagma is given in \Cref{fig-app:exp1-dagma-wt}. 

\begin{figure}[H] %
    \centering
    \subfigure[Results on ER1 data]{
        \includegraphics[width=.34\textwidth]{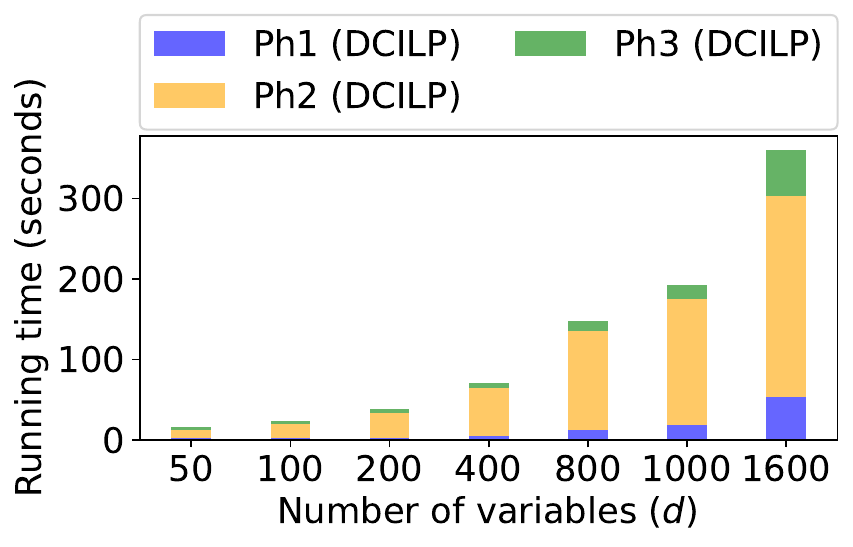} \qquad
        \includegraphics[width=.34\textwidth]{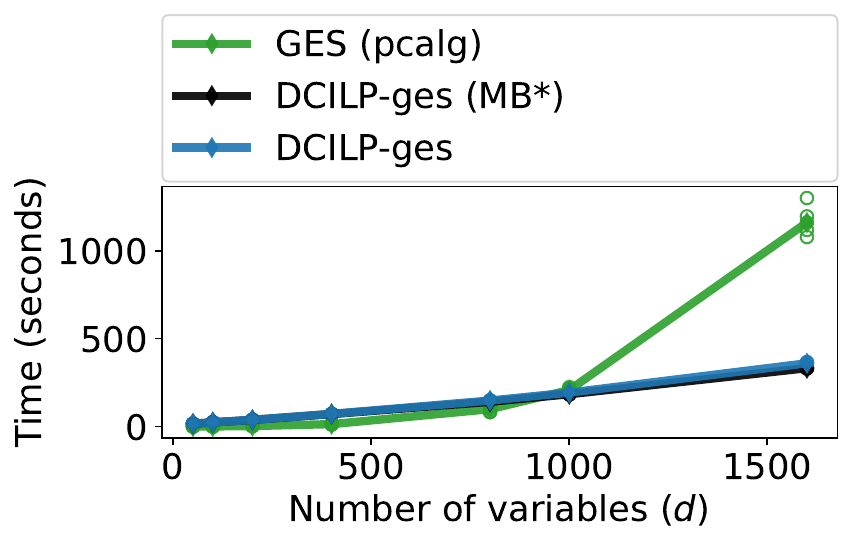}
    }
    \\
    \subfigure[Results on ER2 data]{  
        \includegraphics[width=.34\textwidth]{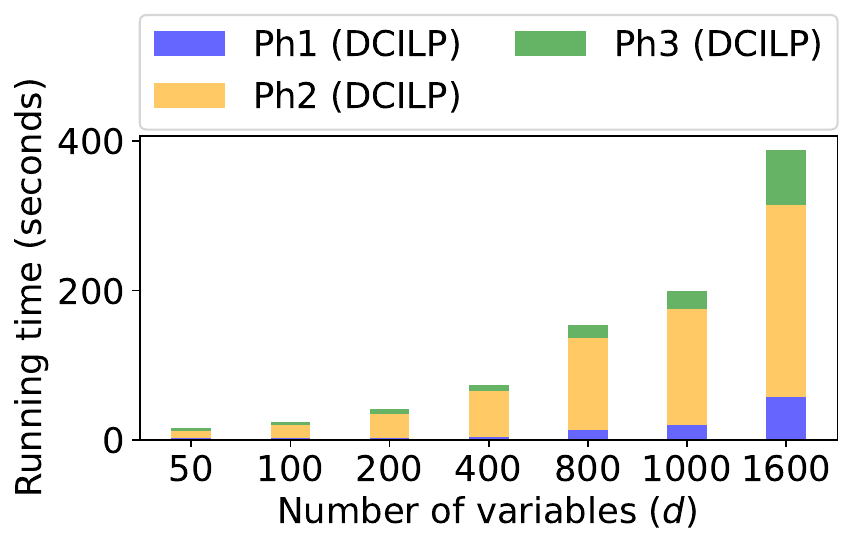}\qquad
        \includegraphics[width=.34\textwidth]{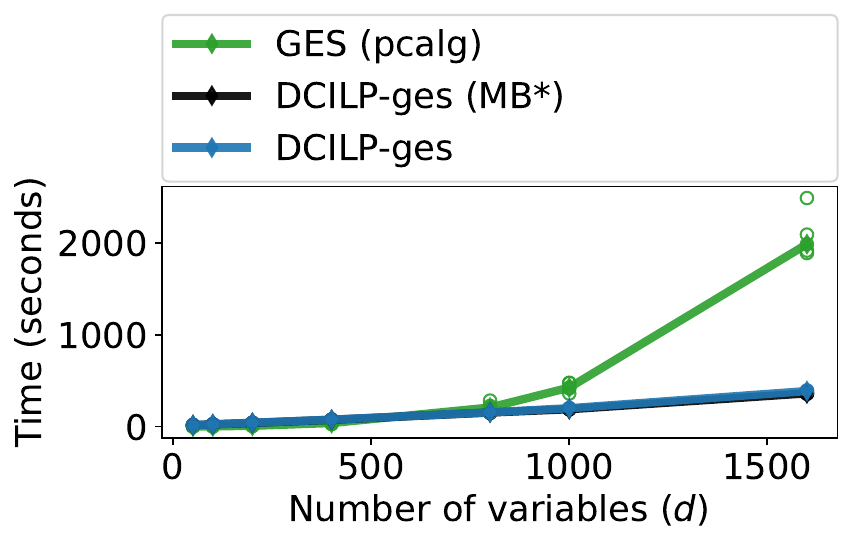}
    }
    \caption{Wall time of \XX-GES and GES for learning linear SEMs on ER1 and ER2 graphs. Left: wall time of the three phases
    of \XX-GES in linear scale; Right: wall time of the two algorithms depending on $d$.} 
    \label{fig-app:exp1-dcges-wt}
\end{figure}

\begin{figure}[H]
    \centering
    \subfigure[Results on ER1 data]{ 
        \includegraphics[width=.34\textwidth]{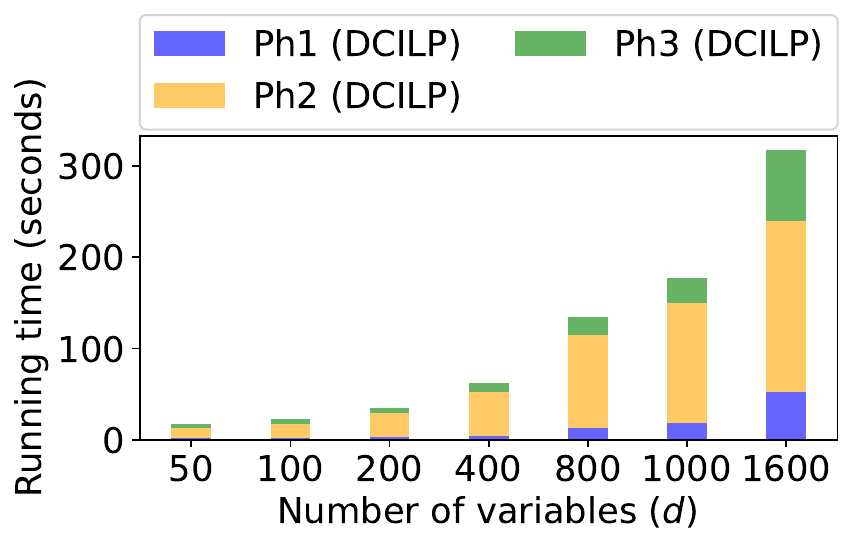} \qquad
        \includegraphics[width=.34\textwidth]{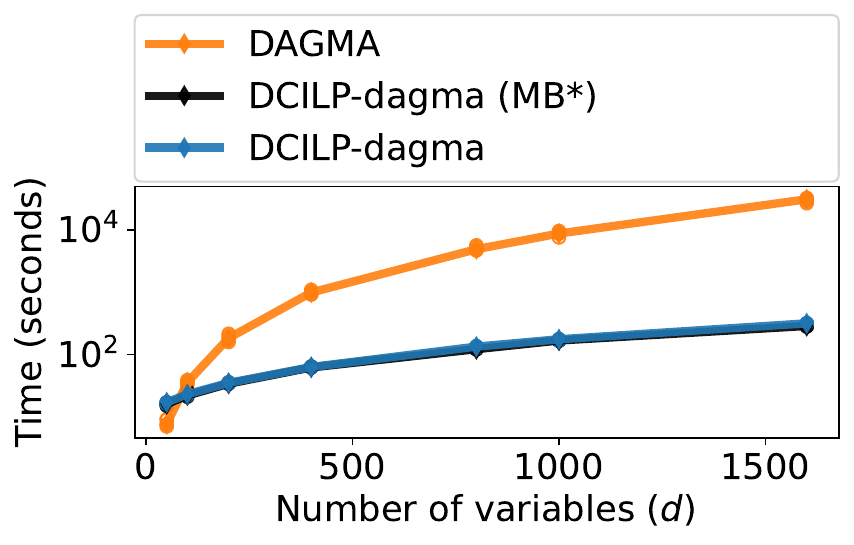}
    }
    \\
    \subfigure[Results on ER2 data]{ 
        \includegraphics[width=.34\textwidth]{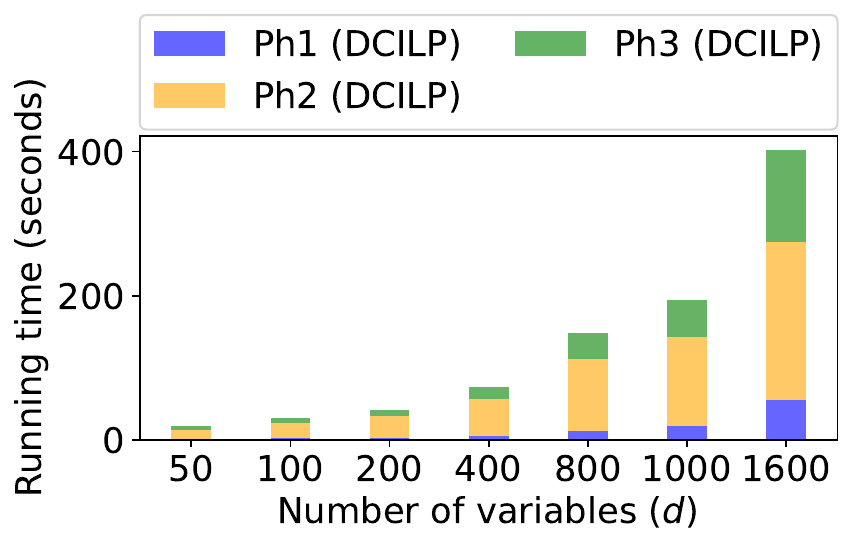} \qquad
        \includegraphics[width=.34\textwidth]{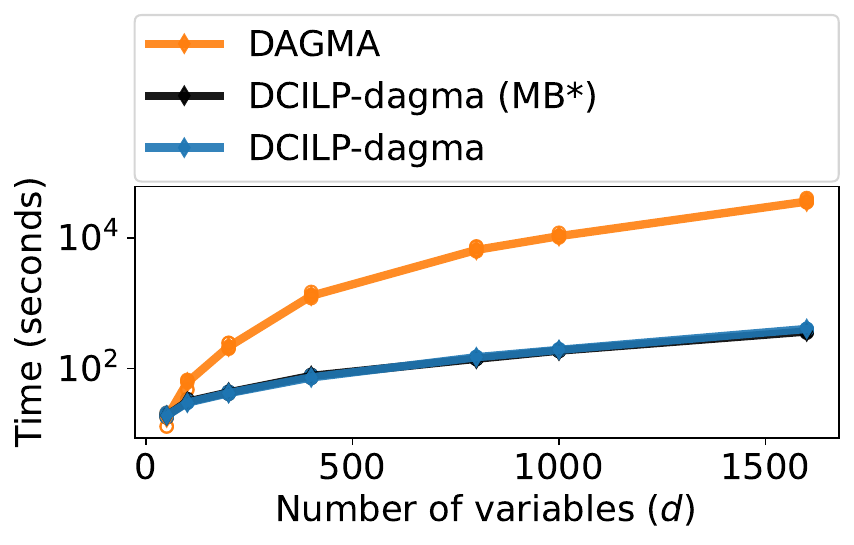}
    }
    \caption{Wall time of \XX-dagma and DAGMA for learning
    the linear SEM on ER1 and ER2 graphs. Left: wall time of the three phases
    of \XX-dagma in linear scale; Right: wall time (in log scale) of the two algorithms depending on $d$.}
    \label{fig-app:exp1-dagma-wt}
\end{figure}

\subsection{Results on SF3 dataset}
\label{ssec-app:sf3}

The results in \Cref{tab-app:sf3-one} show that \XX-dagma yields a relevant trade-off between performance and computational cost. Interestingly, the degradation due to the estimation of the Markov blankets is higher for $d=400$ than for $d=800$, which is attributed to better local effective $n/d$ ratios of the subproblems. 

\begin{table}[htpb]
\small 
\centering
\caption{Causal discovery results with Gaussian SEM data on SF3 graphs, with ratio samples versus dimension $n/d$ set to $20$. Each result is the median out of 3 independent runs (distinct random seeds), with nnz the overall number of edges in the estimated causal graph. } \label{tab-app:sf3-one}
\vspace{1mm}
\begin{tabular}{c|l||ccccr}
\hline
     $d$                                                                              & Algorithm            &  TPR        & FDR         &  SHD           & \# of Edges      & time (seconds)                                \\ 
\hline
\multicolumn{1}{c|}{\multirow{5}{*}{400}}     & GOLEM               & 0.987    	& 0.025   	& 46 {\scriptsize$\pm$ 14.47}   	& 1208          & 4437.526      \\ %
\multicolumn{1}{c|}{                   }    & DAGMA               & 0.999    	& 0.003   	& 3 {\scriptsize$\pm$ 1.53}   	    & 1196          & 3997.719        \\ 
\multicolumn{1}{c|}{                   }    & GES                 & 0.907    	& 0.471   	& 1029 {\scriptsize$\pm$ 237.57}   	& 2048          & 1330.314        \\ 
\multicolumn{1}{c|}{                   }    & DCILP-dagma (MB*)   & 0.964    	& 0.135   	& 208 {\scriptsize$\pm$ 24.43}   	& 1330        	& 1102.636        \\ 
\multicolumn{1}{c|}{                   }    & DCILP-dagma         & 0.925    	& 0.177   	& 322 {\scriptsize$\pm$ 68.83}   	& 1336        	& 974.538         \\ 
\hline
\multicolumn{1}{c|}{\multirow{5}{*}{800}}     & GOLEM               & 0.964    	& 0.049   	& 187 {\scriptsize$\pm$ 355.73}   	& 2426         & 25533.171   \\ %
\multicolumn{1}{c|}{                   }    & DAGMA               & 0.999    	& 0.003   	& 7 {\scriptsize$\pm$ 8.66}   	    & 2398         & 22952.589     \\ 
\multicolumn{1}{c|}{                   }    & GES                 & 0.878    	& 0.518   	& 2444 {\scriptsize$\pm$ 1182.74}   & 4068         & 12063.320     \\ 
\multicolumn{1}{c|}{                   }    & DCILP-dagma (MB*)   & 0.938    	& 0.250   	& 879 {\scriptsize$\pm$ 234.64}   	& 2994         & 4104.407      \\ 
\multicolumn{1}{c|}{                   }    & DCILP-dagma         & 0.923    	& 0.250   	& 914 {\scriptsize$\pm$ 191.78}   	& 2932         & 3405.801      \\ 
\hline
\end{tabular}
\end{table}

\subsection{Complementary results on the MUNIN dataset }
\label{ssec-app:expe-munin}

The comparative results displayed in \Cref{fig-app:munin-u5} and \Cref{fig-app:munin-u10} illustrate the impact of the $n/d$ ratio under the three different SEM noise types (Gaussian, Gumbel, Uniform). All reported results correspond to the hyper-parameter configuration selected with respect to the lowest DAGness value and the lowest number of edges among the final solutions; see \Cref{ssec:ilp-formu}. 

\begin{figure}[htpb]
           \centering
  \subfigure[Gaussian noise]{\includegraphics[width=.30\textwidth]{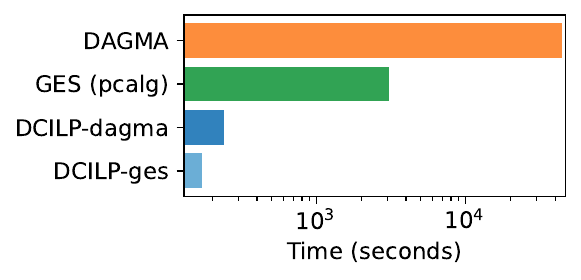} 
   \includegraphics[width=.30\textwidth]{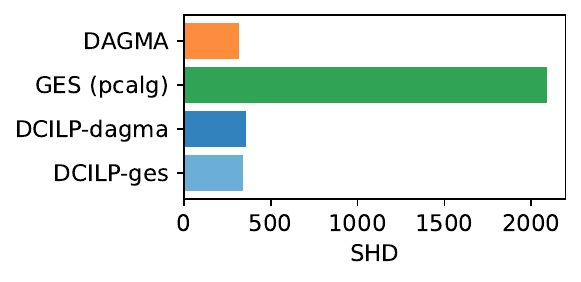}
   \includegraphics[width=.30\textwidth]{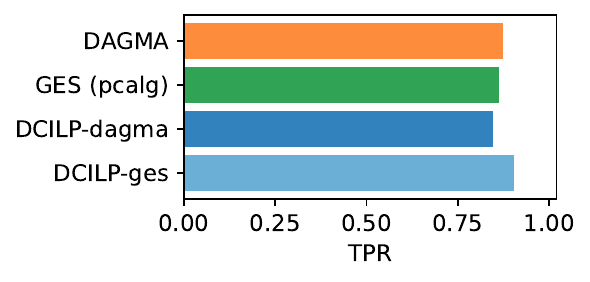}
   }
\subfigure[Gumbel noise]{\includegraphics[width=.30\textwidth]{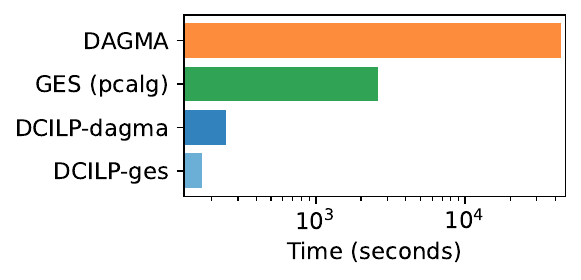} 
   \includegraphics[width=.30\textwidth]{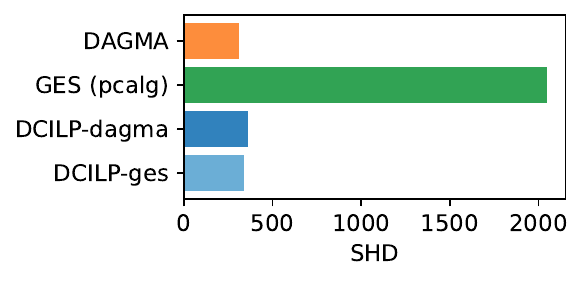}
   \includegraphics[width=.30\textwidth]{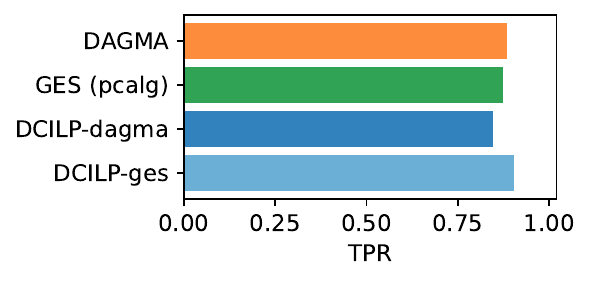}
   }
  \subfigure[Uniform noise]{\includegraphics[width=.30\textwidth]{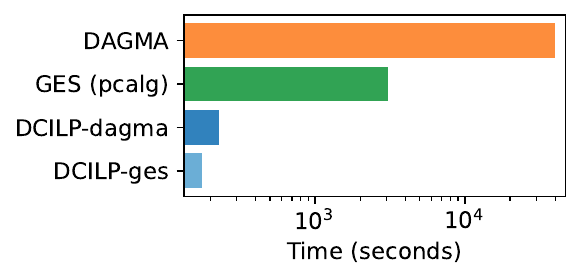} 
   \includegraphics[width=.30\textwidth]{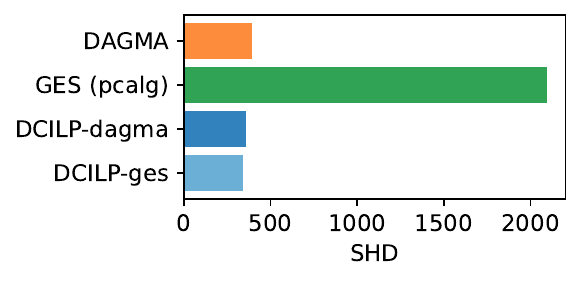}
   \includegraphics[width=.30\textwidth]{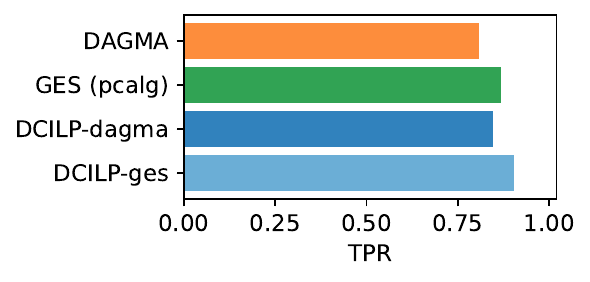}
   }
\caption{Comparative results on MUNIN linear SEM with Gaussian (a), Gumbel (b) and Uniform (c) noise distributions. The ratio $n/d$ is 5.}
        \label{fig-app:munin-u5}
\end{figure}

\begin{figure}[htpb]
           \centering
  \subfigure[Gaussian
  noise]{\includegraphics[width=.30\textwidth]{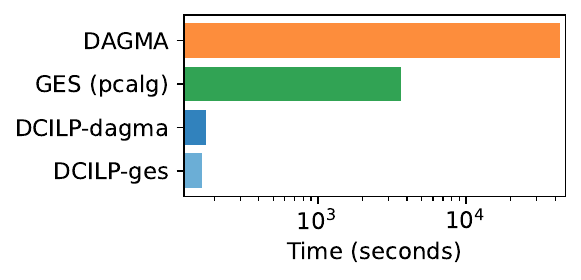} 
   \includegraphics[width=.30\textwidth]{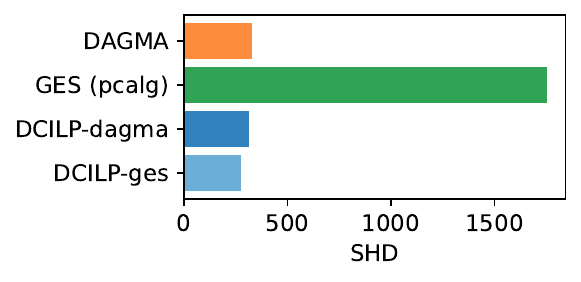}
   \includegraphics[width=.30\textwidth]{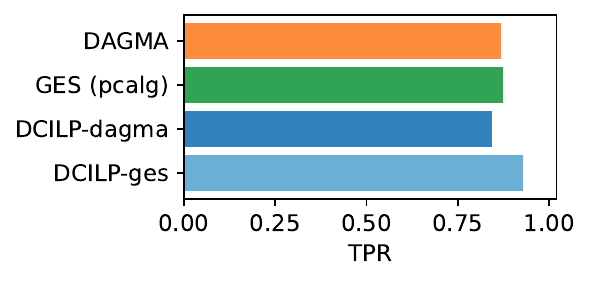}
   }
   \\
  \subfigure[Gumbel noise]{\includegraphics[width=.30\textwidth]{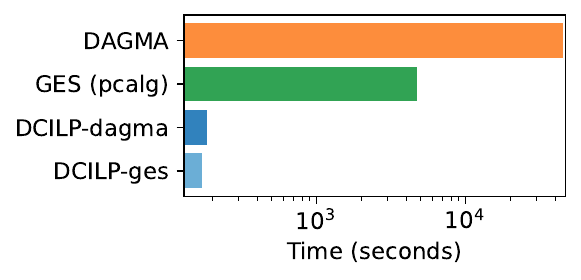} 
   \includegraphics[width=.30\textwidth]{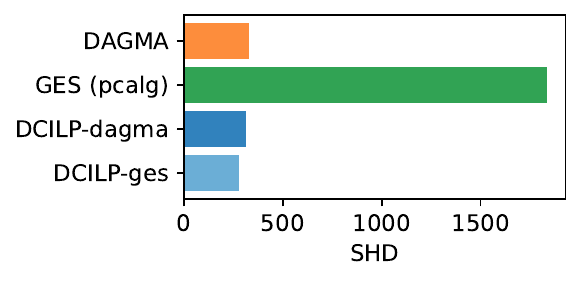}
   \includegraphics[width=.30\textwidth]{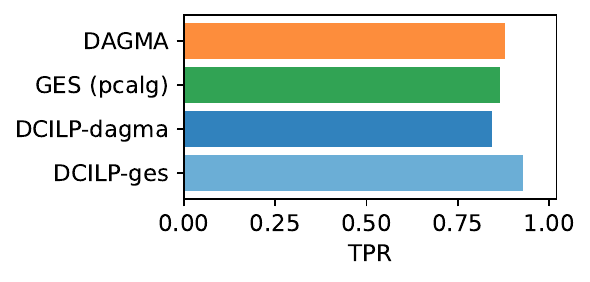}
   }
   \\
  \subfigure[Uniform noise]{\includegraphics[width=.30\textwidth]{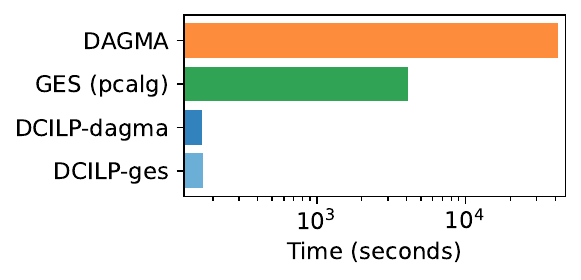} 
   \includegraphics[width=.30\textwidth]{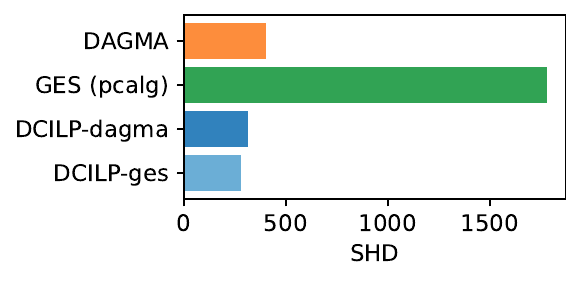}
   \includegraphics[width=.30\textwidth]{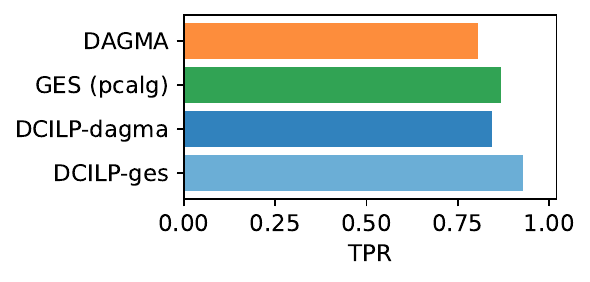}
   }
\caption{Comparative results on MUNIN linear SEM with Gaussian (a), Gumbel (b) and Uniform (c) noise distributions. The ratio $n/d$ is 10.}
\label{fig-app:munin-u10}
\end{figure}

\end{document}